\documentclass{article}
\pdfoutput=1
\usepackage{arxiv}
\setcitestyle{authoryear,open={(},close={)}}
\renewcommand{\cite}{\citep}

\usepackage[utf8]{inputenc} %
\usepackage[T1]{fontenc}    %
\usepackage{hyperref}       %
\usepackage{url}            %
\usepackage{booktabs}       %
\usepackage{amsfonts}       %
\usepackage{nicefrac}       %
\usepackage{microtype, color}      %

\usepackage{tabulary}
\usepackage{tabularx}
\usepackage{array}
\usepackage{algorithm, algorithmicx}
\usepackage[noend]{algpseudocode}
\usepackage{amsmath}
\usepackage{amssymb}
\usepackage{amsthm}
\usepackage{bbm}
\usepackage{graphicx}
\usepackage{indentfirst}

\newlength\savewidth

\newtheorem{theorem}{Theorem}[section]

\newtheorem{definition}[theorem]{Definition}
\newtheorem{lemma}[theorem]{Lemma}

\newtheorem{proposition}[theorem]{Proposition}

\newtheorem{example}[theorem]{Example}
\newtheorem{assumption}[theorem]{Assumption}

\makeatletter
\newtheorem*{rep@definition}{\rep@title}
\newcommand{\newrepdefinition}[2]{%
	\newenvironment{rep#1}[1]{%
		\def\rep@title{#2 \ref{##1}}%
		\begin{rep@definition}}%
		{\end{rep@definition}}}
\makeatother

\newrepdefinition{definition}{Definition}
\newrepdefinition{lemma}{Lemma}
\newrepdefinition{proposition}{Proposition}

\usepackage{amsmath,amsfonts,bm}

\def\eqref#1{equation~\ref{#1}}

\def\1{\bm{1}}

\def\eps{{\epsilon}}

\DeclareMathAlphabet{\mathsfit}{\encodingdefault}{\sfdefault}{m}{sl}
\SetMathAlphabet{\mathsfit}{bold}{\encodingdefault}{\sfdefault}{bx}{n}

\newcommand{\pdata}{p_{\rm{data}}}

\newcommand{\R}{\mathbb{R}}

\DeclareMathOperator*{\argmin}{arg\,min}

\DeclareMathOperator{\Tr}{Tr}

\newcommand{\aug}{\mathcal{A}}
\newcommand{\real}{\mathbb{R}}
\newcommand{\convf}{\mathcal{F}_{\textup{conv}}}
\newcommand{\linearf}{\mathcal{F}_{\textup{linear}}}
\newcommand{\unif}{\mathcal{F}_{\textup{uni}}}
\newcommand{\lipf}{\mathcal{F}_{\textup{Lip}, \kappa}}
\newcommand{\loss}{\mathcal{L}}
\newcommand{\eloss}{\widehat{\mathcal{L}}}
\newcommand{\reluf}{\mathcal{F}_{\text{ReLU}}}
\newcommand{\data}{\mathcal{X}}
\newcommand{\fclass}{\mathcal{F}}
\newcommand{\id}{\textup{id}}
\newcommand{\Exp}{\mathbb{E}}
\newcommand{\norm}[1]{\left\lVert#1\right\rVert}
\newcommand{\ppos}{p_{\rm{pos}}}

\newcommand{\feig}{f_{\text{eig}}}
\newcommand{\identity}{\mathbb{I}}
\newcommand{\pmin}{P_{\text{min}}}
\newcommand{\pmax}{P_{\text{max}}}
\newcommand{\pif}{\pi_f}
\newcommand{\sgn}{\text{sgn}}
\newcommand{\binary}{\text{bin}}
\newcommand{\laplacian}{\mathbb{L}}
\newcommand{\tlaplacian}{\widetilde{\mathbb{L}}}
\newcommand{\rad}[1]{\widehat{\mathcal{R}}_{#1}}
\newcommand{\npre}{{n_\textup{pre}}}
\newcommand{\nds}{{n_\textup{ds}}}
\newcommand{\ncluster}{r}
\newcommand{\nclustertrue}{{r_0}}
\newcommand{\eExp}{\widehat{\mathbb{E}}}

\def\shownotes{1}  \ifnum\shownotes=1
\newcommand{\authnote}[2]{{[#1: #2]}}
\else
\newcommand{\authnote}[2]{}
\fi

\def\shownotes{1}  \ifnum\shownotes=1
\newcommand{\authornotenonurgent}[2]{{[#1: #2]}}
\else
\newcommand{\authornotenonurgent}[2]{}
\fi

\ifdefined\usebigfont

\usepackage{times}
\usepackage[fontsize=13pt]{scrextend}
\AtBeginDocument{
	\newgeometry{left=1.56in,right=1.56in,top=1.71in,bottom=1.77in}
}
\else
\fi

\begin{document}

\begin{center}
	{\LARGE A Theoretical Study of Inductive Biases in Contrastive Learning} \\
	\vspace{.8cm}
	{\large Jeff Z. HaoChen ~~~~ Tengyu Ma} \\
	\vspace{.4cm}
	{\large Stanford University} \\
	\vspace{.05cm}
	Department of Computer Science \\
	\vspace{.4cm}
	\texttt{\{jhaochen, \,tengyuma\}@cs.stanford.edu}
	\vspace{1cm}
\end{center}

\begin{abstract}%
Understanding self-supervised learning is important but challenging. Previous theoretical works study the role of pretraining losses, and view neural networks as general black boxes. However, the recent work of \citet{saunshi2022understanding}  argues that the model architecture --- a component largely ignored by previous works --- also has significant influences on the downstream performance of self-supervised learning. In this work, we provide the first theoretical analysis of self-supervised learning that incorporates the effect of inductive biases originating from the model class. In particular, we focus on contrastive learning --- a popular self-supervised learning method that is widely used in the vision domain. We show that when the model has limited capacity, contrastive representations would recover certain special clustering structures that are compatible with the model architecture, but ignore many other clustering structures in the data distribution. As a result, our theory can capture the more realistic setting where contrastive representations have much lower dimensionality than the number of clusters in the data distribution. We instantiate our theory on several synthetic data distributions, and provide empirical evidence to support the theory.
\end{abstract}

\section{Introduction}

Recent years have witnessed the effectiveness of pre-trained representations, which are learned on unlabeled data with self-supervised losses and then adapted to a wide range of downstream tasks~\citep{chen2020simple, chen2020big, he2020momentum, caron2020unsupervised, chen2020improved,gao2021simcse, su2021tacl,chen2020exploring, brown2020language, radford2019language}. However, understanding the empirical success of this emergent pre-training paradigm is still challenging. It requires novel mathematical frameworks and analyses beyond the classical statistical learning theory. The prevalent use of deep neural networks in self-supervised learning also adds to the mystery. 

Many theoretical works focus on isolating the roles of self-supervised losses, showing that they encourage the representations to capture certain structures of the unlabeled data that are helpful for downstream tasks~\citep{arora2019theoretical, haochen2021provable, haochen2022beyond, wei2021why, xie2021explanation,saunshi2020mathematical}.  However, these works oftentimes operate in the sufficient pre-training data (polynomial in the dimensionality) or even infinite pre-training data regime, and view the neural network as a \textit{black box}. The only relevant property of neural networks in these works is that they form a parameterized model class with finite complexity measure (e.g., Rademacher complexity).

Recently, ~\citet{saunshi2022understanding} argue that the pre-training loss is \textit{not} the only contributor to the performance of self-supervised learning, and that previous works which view neural networks as a black box cannot tell apart the differences in downstream performance between architectures (e.g., ResNet~\citep{he15deepresidual} vs vision transformers~\citep{dosovitskiy2020image}). Furthermore, self-supervised learning with an appropriate architecture can possibly work under more general conditions and/or with fewer pre-training data than predicted by these results on general architecture. Therefore, a more comprehensive and realistic theory needs to take into consideration the inductive biases of architecture.

This paper provides the first theoretical analyses of the inductive biases of \textit{nonlinear} architectures in self-supervised learning. Our theory follows the setup of the recent work by~\citet{haochen2021provable} on contrastive learning and can be seen as a refinement of their results by further characterizing the model architecture's impact on the learned representations.

We recall that \citet{haochen2021provable} show that contrastive learning, with sufficient data and a parameterized model class of finite complexity, is equivalent to spectral clustering on a so-called \emph{population positive-pair graph}, where nodes are augmented images and an edge between the nodes $x$ and $x'$ is weighted according to the probability of encountering $(x, x')$ as a positive pair.
They essentially assume that the positive-pair graph contains several major semantically-meaningful clusters, and prove that contrastive representations exhibit a corresponding clustering structure in the Euclidean space, that is, images with relatively small graph distance have nearby representations. 

Their results highly rely on the clustering property of the graph---the representation dimensionality and pre-training sample complexity both scale in the number of clusters. 
The important recent work of~\citet{saunshi2022understanding}, however, demonstrates with a synthetic setting that contrastive learning can provably work with linear model architectures even if the number of clusters is huge (e.g., exponential in the dimensionality). 
Beyond the simple synthetic example discussed in their paper, there has been no previous work that formally characterizes this effect in a general setting.

In this work, we develop a general theory that leverages the inductive bias to avoid the dependency on the potentially huge number of clusters: although there exists a large number of clusters in the positive-pair graph, the number of clusters \textit{implementable by the model} (which we call \emph{minimal implementable clusters}) could be much smaller, even exponentially. 
Figure~\ref{figure:1} shows an example where a linear function can only implement one clustering structure but not the other, despite both being valid clusters in the positive-pair graph. It’s possible that a minimal implementable cluster consists of multiple well-separated sub-clusters but none of these sub-clusters can be implemented by the model class.

\begin{figure}
	\centering
		\includegraphics[width=0.8\textwidth]{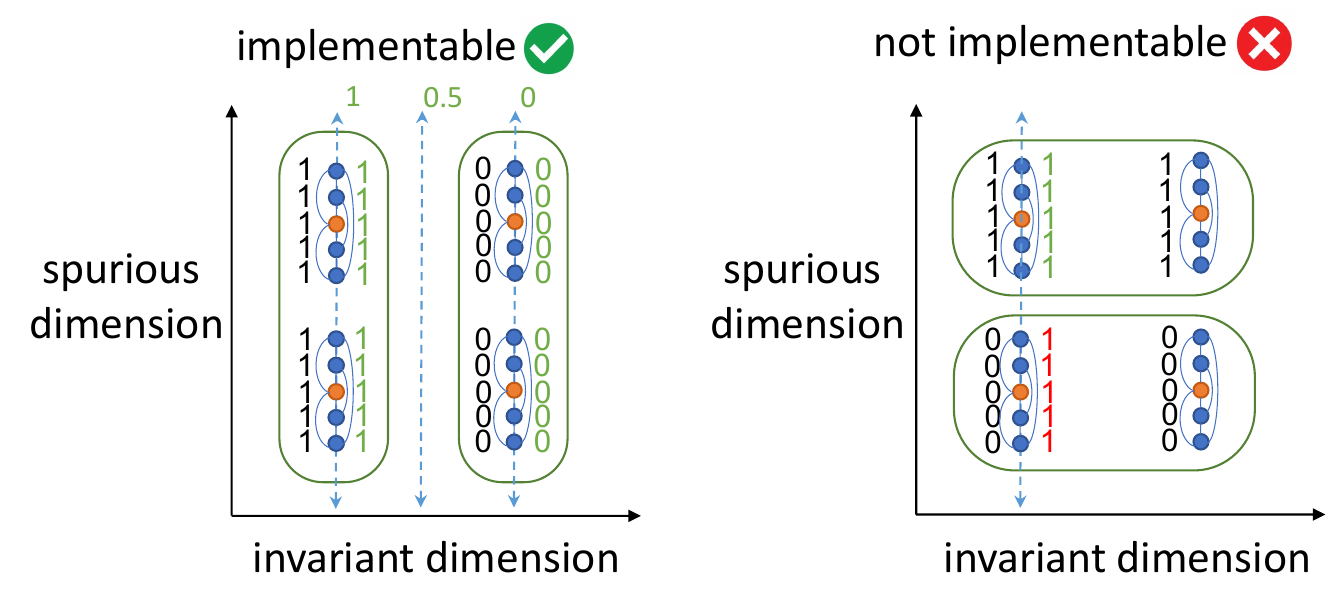}
	\caption{
		\textbf{A simple example where the linear function class learns the correct feature and ignores the spurious feature.} A simplified version of the synthetic example proposed in the work of \citet{saunshi2022understanding}. The orange points are the original data and blue points are augmented data (obtained by adding noise in the spurious dimension). 
The dimension invariant to augmentation is desired. Edges represent positive pairs that are constructed from augmentation. 
We say a real-valued function \emph{implements} a cluster if it outputs $1$ on the cluster and outputs $0$ on all other data.	
We note that here implementing means matching the \emph{exact} value, rather than simply matching the label after applying some linear threshold.
The figure above shows two possible ways to partition the data into two clusters, but only the one on the left-hand side (which captures the invariant dimension) is implementable by a linear function. Here we use black numbers to indicate the target output on the data, and green numbers to indicate the output of the implementing function which extrapolates outside of the data support. Note that linear model is \textit{not} composed with a threshold function.  The partition on the right hand side is not implementable because any linear model that outputs constant $1$ on the upper-left small cluster would also output $1$ on the bottom-left small cluster due to linear extrapolation. Here we use red numbers to indicate the output of the linear function that contradicts with the target.
	}\label{figure:1}	
\end{figure}

We show that contrastive representations would only recover the clustering structures that are compatible with the model class, hence low-dimensional contrastive learned representations would work well on the downstream tasks.
Concretely, suppose the number of minimal implementable clusters is $\ncluster$ which can be much smaller than the number of natural clusters in the graph $\nclustertrue$. \citet{haochen2021provable} prove the efficacy of contrastive learning assuming the representation dimensionality (hence also sample complexity) is larger than $\nclustertrue$.
We develop a new theory (Theorem~\ref{theorem:main_theorem_with_clustering}) that makes the representation dimensionality only depend on $\ncluster$ instead of $\nclustertrue$. 
We also extend this result to a more complex setting where we can deal with even more structured clusters, e.g., when there are $2^{s}$ clusters with certain geometric structures, but the representation dimensionality can scale with only $s$ instead of $2^s$.  See Theorem~\ref{theorem:thm_eigenspace} and its instantiation on Example~\ref{example1} for this result.

We instantiate our theory on several synthetic data distributions and show that contrastive learning with appropriate model architectures can reduce the representation dimensionality, allowing better sample complexity.  We consider a data distribution on a hypercube first proposed by~\citet{saunshi2022understanding} which contains a small subspace of features that are invariant to data augmentation and a large subspace of spurious features.
When the function class is linear, we show that the contrastive representations can solve downstream binary classification tasks if the downstream label only depends on one dimension of  invariant features (Theorem~\ref{theorem:example1}).
When the function class is ReLU networks (hence more expressive), we show that the contrastive representations can solve more diverse downstream classification problems where the label can depend on all invariant features (Theorem~\ref{theorem:example2}).  We also provide examples for Lipschitz-continuous function classes (Theorem~\ref{theorem:example3}) and convolutional neural networks (Theorem~\ref{theorem:example4}).

We provide experimental results to support our theory. We propose a method to test the number of implementable clusters of ResNet-18 on the CIFAR-10 dataset and show that there are indeed only a small number of implementable clusters under the model architecture constraint (Section~\ref{section:experiments}).

\section{Related works}
Contrastive learning learns representations from different views or augmentations of inputs~\citep{chen2020simclr, hjelm2018learning, wu2018unsupervised, tian2019contrastive, chen2021exploring, gao2021simcse, bachman2019learning, 
oord2018representation, ye2019unsupervised, henaff2020data, misra2020self, caron2020unsupervised, zbontar2021barlow, bardes2021vicreg, tian2020makes, robinson2021contrastive, dubois2022improving}. The learned representation can be used (either directly or after finetuning) to solve a wide range of downstream tasks with high accuracy.

The empirical success of contrastive learning has attracted a series of theoretical works that study the contrastive loss~\citep{arora2019theoretical, haochen2021provable, haochen2022beyond, tosh2020contrastive, tosh2021contrastive, lee2020predicting, wang2021chaos, nozawa2021understanding, ash2022investigating, tian2022deep}, most of which treat the model class as a black box except for the work of \citet{lee2020predicting} which studies the learned representation with linear models, and the works of \citet{tian2022deep} and \citet{wen2021toward} which study the training dynamics of contrastive learning for linear and 2-layer ReLU networks. 
Most related to our work is~\citet{saunshi2022understanding} which theoretically shows (on a linear toy example) that appropriate model classes help contrastive learning by reducing the sample complexity. We generalize their results to broader settings. 

Several theoretical works also study non-contrastive methods for self-supervised representation learning~\citep{wen2022mechanism, tian2021understanding, garrido2022duality, balestriero2022contrastive}. 
\citet{garrido2022duality} establish the duality between contrastive and non-contrastive methods. 
\citet{balestriero2022contrastive} provide a unified framework for contrastive and non-contrastive methods. 
There are also works theoretically studying self-supervised learning in other domains such as language modeling~\citep{wei2021why,xie2021explanation,saunshi2020mathematical}. 
\section{From clusters to minimal implementable clusters}\label{section:main_theory}
In this section, we introduce our main theoretical results regarding the role of inductive biases of architectures in contrastive learning. 
Recall that contrastive learning encourages two different views of the same input (also called a \emph{positive pair}) to have similar representations, while two random views of two different inputs (also called a \emph{negative pair}) have representations that are far from each other.
Formally, we use $\pdata$ to denote the distribution of a random view of random input, use  $\ppos$ to denote the distribution of a random positive pair, and $\data$ to denote the support of $\pdata$. For instance, $\data$ is the set of all augmentations of all images for visual representation learning.

Following the setup of \citet{haochen2022beyond}, for a representation map $f:\data\rightarrow\real^k$ where $k$ is the representation dimensionality, we learn the contrastive representation by minimizing the following generalized spectral contrastive loss:
\begin{align}
	\loss_{\lambda}(f) := \Exp_{(x, x^+)\sim \ppos}[\norm{f(x)-f(x^+)}_2^2] + \lambda \cdot R(f),
\end{align}
where $\lambda>0$ is a hyperparameter indicating the regularization strength, and the regularizer normalizes the representation covariance towards the identity matrix:
\begin{align}
	R(f) := \norm{\Exp_{x\sim\pdata}[f(x)f(x)^\top] - \identity}_F^2.
\end{align}
This loss is very similar to the popular Barlow Twins loss~\citep{zbontar2021barlow} and has been shown to empirically work well~\citep{haochen2021provable}. 
Theoretically, the prior work proposes the notion of \emph{positive-pair graph} with $\data$ being the vertex set and an edge between the nodes $x$ and $x'$ is weighted according to the probability of encountering $(x, x')$ as a positive pair (i.e., $\ppos(x, x')$).
This graph is defined on the \textbf{population} data, and intuitively captures the semantic relationship between different data --- when the positive pairs are formed by applying data augmentation to the same natural data, it is expected that datapoints in the same cluster in the positive-pair graph would have similar semantic meanings. Figure~\ref{figure:posgraph} gives a demonstration of the positive-pair graph. 

\begin{figure}
	\hspace*{0.4in}\includegraphics[width=0.8\textwidth]{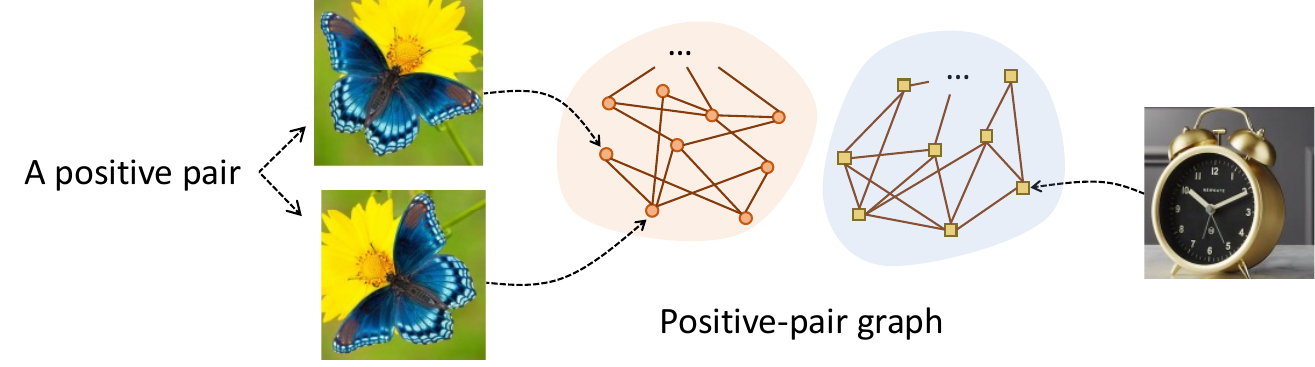}%
	\caption{
		\textbf{A demonstration of the positive-pair graph.} When the positive pairs are formed by applying data augmentation (such as rotation) to the same natural image, data with the same semantic meaning (e.g., the two butterfly images) tend to belong to the same cluster in the positive-pair graph. Datapoints with different semantic meanings (e.g., a butterfly image and a clock image) would not be connected in the positive-pair graph, hence belong to different clusters.	
		\label{figure:posgraph}
	}	
\end{figure}

Their analysis shows that learning contrastive representations with the above loss is equivalent to spectral clustering~\citep{ng2001spectral,shi2000normalized} on this positive-pair graph, hence can learn meaningful representations when the graph has clustering structures.

Different from the prior work, we study the representation map that minimizes the contrastive loss \emph{within a certain function class $\fclass$}. Here we assume functions in $\fclass$ map data in $\data$ to representations in $\real^k$ for some dimensionality $k$.  The main contribution of our result is the improvement of $k$ due to the consideration of this specific function class: by studying the representation learned within a constrained model class $\fclass$, we will show that the necessary representation dimensionality $k$ is much smaller than that required in the prior work. As a result, the sample complexity for the downstream labeled task would be improved compared to the prior work.

Let $\{S_1, S_2, \cdots, S_m\}$ be a $m$-way partition of $\data$, i.e., they are disjoint non-empty subsets of $\data$ such that $\data = \cup_{i\in[m]} S_i$.  For any $x\in\data$, let $\id_x$ be the index such that $x\in S_{\id_x}$. We consider a partition of the graph such that there is not much connection between any two clusters, which is formalized by the following assumption.
 
\begin{assumption}[$\alpha$-separability]\label{assumption:connectivity}
	The probability of a positive pair belonging to two different sets is less than $\alpha$:
	\begin{align}
		\Pr_{(x, x^+)\sim \ppos}(\id_x\ne \id_{x^+})\le \alpha.
	\end{align}
\end{assumption}

We consider downstream tasks that are $r$-way classification problems with label function $y(\cdot): \data\rightarrow[r]$. We assume that the downstream task aligns with the clusters:
\begin{assumption}\label{assumption:downstream}
	The downstream label $y(x)$ is a constant on each $S_i$.
\end{assumption}

Our key assumptions about the function class are that it can implement desirable clustering structures (Assumption~\ref{assumption:implementability}) but cannot break the positive-pair graph into too many clusters (Assumption~\ref{assumption:separability}).

Let $S\subset\data$ be a subset of $\data$, $\pdata^S$ be the distribution $\pdata$ restricted to set $S$, and $\ppos^S$ be the positive pair distribution $\ppos$ conditioned on both datapoints in the pair belonging to set $S$. For any function $g:S\rightarrow\real$, we define the following expansion quantity:
\begin{align}
	Q_S(g) := \frac{\Exp_{(x, x^+)\sim \ppos^S}[(g(x)-g(x^+))^2]}{\Exp_{x\sim\pdata^S, x'\sim\pdata^S}[(g(x)-g(x'))^2]}.\label{eqn:11}
\end{align}
We let $Q_S(g)=\infty$ if the denominator is $0$. 
Here the numerator represents the discrepancy between a random positive pair, and the denominator represents the global variance of $g$. 
Intuitively, a smaller value $Q_S(g)$ means that function $g$ does a better job at separating the set $S$ into disjoint sub-clusters, and hence implements an inner-cluster connection structure that is sparse. For instance, if $S$ contains two disjoint sub-clusters, and $g$ has different constant values on each of them, then $Q_S(g)=0$. On the other hand, if $S$ is densely connected, then $Q_S(g)>0$ regardless of the choice of $g$. We also note that $Q_S(\cdot)$ is also closely related to the sparsest cut formulation in spectral graph theory. When $g$ is restricted to output only values in $\{0,1\}$, then the RHS of~\eqref{eqn:11} is the sparsest cut value of the subgraph supported on the vertices in $S$ (cf. Definition 4 or equation 2.3 of~\citet{Trevisan2015NotesOE}), and $Q_S$ is also a typical way to relax the sparsest cut value (cf. Section 2.3 of~\citet{Trevisan2015NotesOE}). 

The first assumption about the function class $\fclass$ assumes that no function in the class can break one cluster into two well-separated sub-clusters:
\begin{assumption}[$\fclass$-implementable inner-cluster connection larger than $\beta$]\label{assumption:separability}
	For any function $f\in\fclass$ and any linear head $w\in \real^k$, let function $g(x) = w^\top f(x)$. For any $i\in[m]$ we have that:
	\begin{align}
		Q_{S_i}(g)\ge \beta.
	\end{align}
\end{assumption}

We note that when the function class $\fclass$ contains \emph{all} the functions from $\data$ to $\real^k$, Assumption~\ref{assumption:separability} essentially says that each of $\{S_1, S_2, \cdots, S_m\}$ has large internal expansion, hence recovers Assumption 3.5 in \citet{haochen2021provable}.
However, when $\fclass$ has limited capacity, each cluster $S_i$ can still contain well-separated sub-clusters, but just those sub-clusters cannot be implemented by functions in $\fclass$.

Assumption~\ref{assumption:separability}  implies that the function class cannot be too expressive. However, in order for the learned representation map to be useful for downstream tasks, it needs to be expressive enough to represent the useful information. Thus, we introduce the following assumption on the function class. 
\begin{assumption}[Implementability]\label{assumption:implementability}
	Recall that $\id_x$ is the index such that $x\in S_{\id_x}$.
	There exists a function $f\in\fclass$ such that $f(x)=e_{\id_x}$ for all $x\in\pdata$ where $e_i\in\real^m$ is the vector where the $i$-th dimension is $1$ and other dimensions are $0$. 
\end{assumption}

When both Assumption~\ref{assumption:separability} and Assumption~\ref{assumption:implementability} hold, we say $\{S_1, S_2, \cdots, S_m\}$ are \emph{minimal implementable clusters} with respect to $\fclass$.

We also introduce the following Assumption~\ref{assumption:closeness_under_scaling} which is true for any function class implemented by a neural network where the last layer is linear. We note that this assumption is needed only for the technical rigour of the proof, and is not essential to the conceptual message of our theory.
\begin{assumption}[Closure under scaling]\label{assumption:closeness_under_scaling}
	For any function $f\in\fclass$ and vector $u\in\real^m$, define function $f'(x) = u\odot f(x)$ where $\odot$ means element-wise product. Then, we have $f'\in\fclass$.
\end{assumption}

Let  $\pmin := \min_{i\in[m]}\Pr_{x\sim\pdata}(x\in S_i)$ and $\pmax := \max_{i\in[m]}\Pr_{x\sim\pdata}(x\in S_i)$ be the sizes of the smallest and largest sets respectively.
Under the above assumptions, we have the following theorem that shows learning a representation map within $\fclass$ and representation dimensionality $k=m$ can solve the downstream task:
\begin{theorem}\label{theorem:main_theorem_with_clustering}
	Suppose $\{S_1, S_2, \cdots, S_m\}$ are minimal implementable clusters with respect to $\fclass$ (i.e., Assumptions~\ref{assumption:connectivity} and~\ref{assumption:separability} hold), and the function class $\fclass$ satisfies Assumptions~\ref{assumption:implementability} and~\ref{assumption:closeness_under_scaling}. For $\lambda>\alpha/\pmin$, consider a learned representation map $\hat{f} = \argmin_{f \in F} \loss_\lambda(f)$ that minimizes the contrastive loss. 
	Then, when $k=m$, for any downstream task that satisfies Assumption~\ref{assumption:downstream}, there exists a linear head $W\in\real^{r\times k}$ which achieves downstream error 
	\begin{align}
		\Exp_{x\sim\pdata}\big[\big\|{W\hat{f}(x)-e_{y(x)}}\big\|_2^2\big]\le  \frac{\alpha}{\beta} \cdot \frac{\pmax}{\pmin-\alpha}.
	\end{align}
\end{theorem}
We note that $\pmax\approx\pmin$ when the partitions are balanced. Thus, so long as $\alpha\ll \pmin$ (i.e., the probability of a positive pair crossing different clusters is smaller than the probability of it containing data from the smallest cluster), the right-hand side is roughly ${\alpha}/{\beta}$. Thus, when the inter-cluster connection $\alpha$ is smaller than the inner-cluster connection that is implementable by the function class $\beta$, the downstream accuracy would be high.

\noindent\textbf{Comparison with \citet{haochen2021provable}.} We note that our result requires $k=m$, whereas \citet{haochen2021provable} provide analysis in a more general setting for arbitrary $k$ that is large enough. Thus, when the function class $\fclass$ is the set of all functions, our theorem recovers a special case of \citet{haochen2021provable}. Our result requires a stricter choice of $k$ mainly because when $\fclass$ has limited capacity, a higher dimensional feature may contain a lot of ``wrong features'' while omitting the ``right features'', which we discuss in more details in the next section.

\section{An eigenfunction viewpoint}\label{section:eigenspace}

In this section, we introduce an eigenfunction perspective that generalizes the theory in the previous section to more general settings. We first introduce the background on eigenfunctions and discuss their relation with contrastive learning. Then we develop a theory that incorporates the model architecture with assumptions stated using the language of eigenfunctions. The advantage over the previous section is that we can further reduce the required representation dimensionality when the minimal implementable clusters exhibit certain internal structures. 

Here we note that we use the language of eigenfunctions because the positive-pair graph can be infinite. Casual readers can think of the graph as a very large but finite graph and treat all the eigenfunctions as eigenvectors.

Our theory relies on the notion of \emph{Laplacian operator} $\laplacian$ of the positive-pair graph, which maps a function $g: \data\rightarrow\real$ to another function $\laplacian(g): \data\rightarrow\real$ defined as follows.
\begin{align}
	\laplacian(g)(x) := g(x) - \int\frac{\ppos(x, x')}{{\pdata(x)}} g(x')dx'.
\end{align}
We say a function $g$ is an eigenfunction of $\laplacian$ with eigenvalue $\psi\in\real$ if for some scalar $\psi$, 
\begin{align}
	\Exp_{x\sim\pdata}\left[(\psi\cdot g(x)- \laplacian(g)(x))^2\right]=0.
\end{align}
This essentially means that $L(g) = \psi \cdot g$ on the support of $\pdata$. 

\noindent\textbf{Eigenfunctions with small eigenvalues achieve small loss on the positive-pairs.} 
One important property is that when $g$ is an eigenfunction with small eigenvalue $\psi$, the quadratic form $\Exp_{(x, x^+)\sim \ppos}[(g(x) - g(x^+))^2]$ is also small, that is, there is a good match between the positive-pairs. In particular, when $\psi=0$, we have that $\laplacian(g)(x) = 0$ on the support of $\pdata$. Thus, 
\begin{align}
	\Exp_{(x, x^+)\sim \ppos}\left[(g(x) - g(x^+))^2\right] &= 2\Exp_{x\sim\pdata}\left[g(x)^2\right] - 2\Exp_{(x, x^+)\sim \ppos}\left[g(x)g(x^+)\right]\\
	&= 2\Exp_{x\sim\pdata}\left[g(x)^2\right] - 2\Exp_{x\sim\pdata}\left[g(x)\int \frac{\ppos(x, x')}{{\pdata(x)}} g(x')dx'\right]\\
	&= 2\Exp_{x\sim\pdata}\left[g(x)^2\right] - 2\Exp_{x\sim\pdata}\left[g(x)^2\right] = 0.
\end{align}

We can formalize this property of eigenfunctions with the following proposition.
\begin{proposition}\label{proposition:eigenvector}
	Any function $g:\data\rightarrow\real$ that satisfies $\Exp_{(x, x^+)\sim \ppos}\left[(g(x) - g(x^+))^2\right]=0$ is an eigenfunction of $\laplacian$ with eigenvalue $0$.
\end{proposition}

\noindent\textbf{Clusters can be represented by small eigenfunctions.}
Intuitively, small eigenfunctions (i.e., eigenfunctions with small eigenvalues) correspond to disconnected clusters in the positive-pair graph. 
To see this, let $S$ be a cluster that is disconnected from the rest of the positive-pair graph. Let $g$ implement the indicator function of $S$, i.e.,  $g(x) = 1$ if $x\in S$, and $g(x)=0$ if $x\notin S$. Since $S$ is disconnected from the rest of the graph, $\ppos(x, x')$ is non-zero only if $x$ and $x'$ both belong to $S$ or both are not in $S$. As a result, one can verify that $\laplacian(g)(x)=0$ for all $x$, thus $g$ is an eigenfunction with eigenvalue $0$. More generally, eigenfunctions with small but non-zero eigenvalues would correspond to clusters that are almost disconnected from the rest of the graph. This correspondence is well-known in the spectral graph theory literature~\citep{trevisan2017lecture}.

\noindent\textbf{Eigenfunctions can capture more information than clustering.} 
Although clusters correspond to eigenfunctions with small eigenvalues, eigenfunctions could be more coarse-grained than clusters---e.g., one eigenfunction could represent two clusters because the sum of indicator functions of two disconnected clusters is also an eigenfunction. Therefore, when the clusters have some geometric structures with each other, the eigenfunctions could help capture them. For instance, consider the situation where there are $4$ disconnected clusters centered at a 2 by 2 grid, e.g., $\{0, 1\}^2$. Since each cluster gives an eigenfunction, there is a subspace of 4 eigenfunctions. 
However, two of them are special---the two eigenfunctions that group the four clusters into two groups of two clusters along the axis---in the sense that they are linear, and therefore if the model family is restricted to be linear models, the linear inductive bias will prefer these two eigenfunctions over others. Figure~\ref{figure:lineareig} gives a demonstration of this example. We note that the cluster in the left figure of Figure~\ref{figure:lineareig} is not a minimal implementable cluster under our current definition in Section~\ref{section:main_theory}, since there exists a linear function within the function class that can separate it into two halves with zero connection. The eigenfunction viewpoint allows us to discuss this kind of structure.

\begin{figure}
	\centering
	\includegraphics[width=0.6\textwidth]{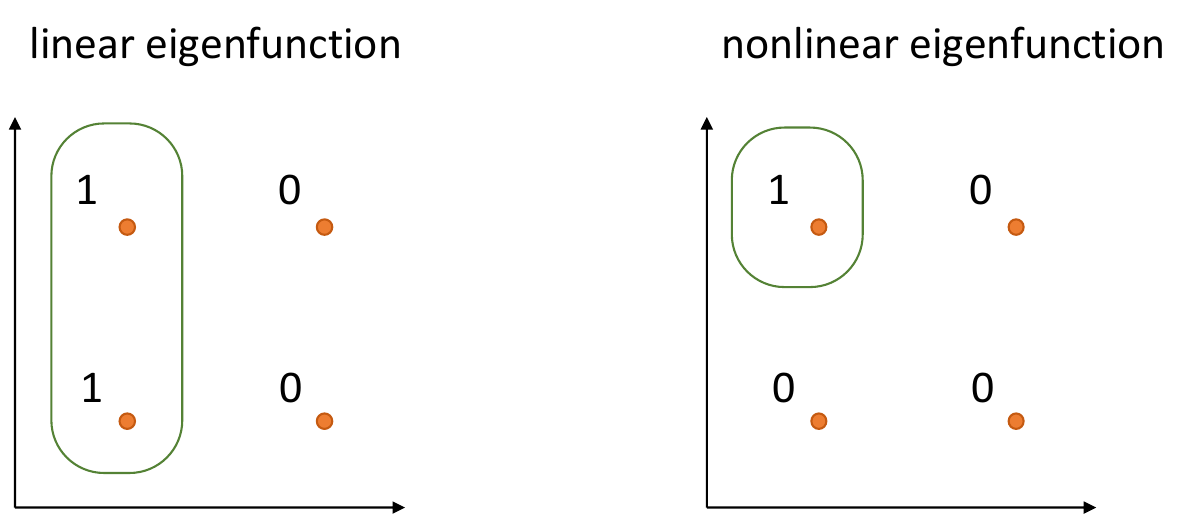}%
	\caption{
		\textbf{Eigenfunctions capture the geometric relationship between clusters.}  Consider four disconnected clusters/nodes in the 2-dimensional Euclidean space.	
		Since each cluster gives an eigenfunction, there is a subspace of 4 eigenfunctions that describes all the possible partitions. 
		However, two of them are special---the two eigenfunctions that group the four clusters into two groups of two clusters along the axis---in the sense that they are linear. Therefore, the eigenfunction viewpoint helps us distinguish these special groupings from the others: when constrained to be linear, an eigenfunction can represent the clustering in the left figure but not that in the right figure. We note that the cluster in the left figure is not a minimal implementable cluster under our current definition in Section~\ref{section:main_theory}, since there exists a linear function within the function class that can separate it into two halves with zero connection.
		\label{figure:lineareig}
	}	
\end{figure}

In this section, we provide a generalized theory based on characterizing the implementability of eigenfunctions. Intuitively, we will assume that there exist $m$ (and only $m$) orthogonal eigenfunctions in the function class with very small eigenvalue, and the downstream task can be solved by these eigenfunctions. More precisely, let $\phi\ge0$ be a very small real number (which can be thought of as $0$ for casual readers). We assume that there exist $m$ approximate eigenfunctions with small eigenvalues. 

\begin{assumption}\label{ass:5}
	Assume that there exist $m$ ``approximate eigenfunctions'', denoted by ${\feig}_i: \data\rightarrow\real, i\in[m]$, which belong to the function class $\fclass$ and satisfy the following conditions: 
	\begin{align}
		\textup{small quadratic form:} & \quad \sum_{i=1}^m \Exp_{(x, x^+)\sim\ppos} \big[\big\|{{\feig}_i(x)-{\feig}_i(x^+)}\big\|_2^2\big] \le \phi. \label{eqn:12} \\
		\textup{unit norm}: & \quad 	\Exp_{x\sim\pdata} \left[{\feig}_i(x)^2\right]=1, ~\forall i\in[m] \label{eqn:13} \\
		\textup{orthogonality}: & \quad \Exp_{x\sim\pdata}\left[{\feig}_i(x){\feig}_j(x)\right] = 0, \forall i\neq j \in [m]\label{eqn:14}
	\end{align}
	Let $\feig(x) = [{\feig}_1(x), {\feig}_2(x), \cdots, {\feig}_m(x)]^\top$ be the vector-output representation map that concatenates the above $m$ functions.
	We can summarize the above equations as for some real number $\phi \ge 0$, there exists $\feig: \data \rightarrow \real^m$ such that 
	\begin{align}
		& \Exp_{(x, x^+)\sim\ppos} \big[\big\|{\feig(x)-\feig(x^+)}\big\|_2^2\big] \le \phi \label{eqn:7}\\
		\textup{and }\quad 	&  \Exp_{x\sim\pdata}\big[\feig(x)\feig(x)^\top\big] = \identity. \label{eqn:8}
	\end{align}
	Moreover, we can define $\phi_m$ be the minimal value of $\phi$ such that there exists $\feig$ that satisfies~\eqref{eqn:7} and~\eqref{eqn:8}. 
\end{assumption}

Assumption~\ref{ass:5} can be viewed as a relaxation of Assumptions~\ref{assumption:connectivity} and~\ref{assumption:implementability} in Section~\ref{section:main_theory}. Reiterating the intuition mentioned in the paragraph with the header ``Clusters can be represented by small eigenfunctions'', we can intuitively demonstrate the existence of approximate eigenfunctions in some special cases. First, consider a graph with $m$ \textit{disconnected} clusters of equal size. Then, letting $\feig(x) = \sqrt{m} \cdot e_{i_x}$ where $i_x$ denotes the index of the cluster that $x$ belongs to, and $e_{i_x}\in \R^m$ is the one-hot vector for index $i_x$, we see that $\feig$ satisfies Assupmtion~\ref{ass:5} with $\phi =0$. In other words, the cluster identity function is a set of approximate eigenfunctions. Second, suppose the positive pair graph contains $m$ clusters with probability mass $p_1, p_2, \cdots, p_m$, then $\feig (x) = p_{i_x}^{-{1}/{2}}\cdot e_{i_x}$ satisfies the constraints with $\phi$ representing the total number of edges between clusters with some weighting that depends on the cluster size.

Next, we state an assumption analogous to Assumption~\ref{assumption:separability} in the previous section. Recall that Assumption~\ref{assumption:separability} intuitively says that even though a larger number of clusters exist in the positive-pair graph, many of them are not implementable by the function class. From the eigenfunction viewpoint, Assumption~\ref{assumption:separability} means that even though there could be many eigenfunctions with small eigenvalues, only a small number of them are in the function class $\fclass$. Concretely, we make the following corresponding assumption which says that the vector-valued function $\feig$ that was assumed to exist in Assumption~\ref{ass:5} spans all the implementable eigenfunctions with small eigenvalue.
\begin{assumption}\label{assumption:no_other_eigenvector_2}
	Let $\feig$ be the approximate eigenfunction in Assumption~\ref{ass:5}. Let $g$ be a real-valued function that is implementable by $\fclass$ (in the sense that $g(x) = f(x)_i$ for some $f\in\fclass$ and $i\in[k]$). We assume that any such $g$ which is an approximate eigenfunction with a small eigenvalue in the sense that 
	\begin{align}
		\Exp_{(x, x^+)\sim\ppos}[(g(x)  - g(x^+))^2] \le \tilde{\phi} \cdot \Exp_{x\sim\pdata}\left[g(x)^2\right] ,
	\end{align}
	must also be a linear combination of $\feig$, that is, there exists $\tilde{w}\in\real^m$ such that  
	\begin{align}
		\Exp_{x\sim\pdata}\left[(\tilde{w}^\top \feig(x) - g(x))^2\right] \le \epsilon.
	\end{align}
	Here both $\tilde{\phi}$ and $\epsilon$ are very small and can be thought of as $0$.
\end{assumption}

We consider downstream tasks that can be solved by $\feig$.
Let $\vec{y}(x)\in\real^r$ be a vector that represents the downstream label of data $x$ (e.g., the one-hot embedding of the label when the downstream task is classification).
We have the following assumption on the downstream task, which is analogous to Assumption~\ref{assumption:downstream} in Section~\ref{section:main_theory}:
\begin{assumption}\label{assumption:downstream_2}
	There exists a linear head $W^*\in \real^{r\times m}$  with norm $\norm{W^*}_F\le B$ such that 
	\begin{align}
		\Exp_{x\sim\pdata}\big[\big\|{W^*\feig(x)-\vec{y}(x)}\big\|_2^2\big]\le \zeta.
	\end{align}
	Here $\zeta$ can be thought of as a very small number.
\end{assumption}

We have the following theorem using the above two assumptions:
\begin{theorem}\label{theorem:thm_eigenspace}
	Under Assumption~\ref{ass:5}, \ref{assumption:no_other_eigenvector_2}, \ref{assumption:downstream_2} (in which $\phi, \tilde{\phi}, \eps, B, \zeta$ are defined), 	suppose $\tilde{\phi}>\phi$ or $\tilde{\phi}=\phi = 0$.
	Then, when $k=m$, for any $\lambda>0$ such that $\phi\le\tilde{\phi}(1-\sqrt{\phi/\lambda})$ and learned representation map $\hat{f} = \argmin_{f \in \fclass} \loss_\lambda(f)$, there exists a linear head $W\in\real^{r\times k}$ such that
	\begin{align}\label{eqn:9}
		\Exp_{x\sim\pdata}\big[\big\|{W\hat{f}(x)-\vec{y}(x)}\big\|_2^2\big]\lesssim  \zeta + B^2k\big(\epsilon+\frac{\phi}{\lambda}\big).
	\end{align}
\end{theorem}

Since $\zeta$, $\epsilon$ and $\phi$ are all very small values, the RHS of \eqref{eqn:9} is very small, hence the learned representation achieves small downstream error. 
When these quantities are exactly $0$, the proof is simple (or arguably trivial), because the learned representation would exactly correspond to the eigenfunctions in the function class with eigenvalue $0$. In this case, the contribution of the theorem is mainly the identification of the proper mathematical assumptions that relate the existence of approximate eigenfunctions with the performance of contrastive learning. 
In the more general case where $\zeta$, $\eps$, and $\phi$ are small but not exactly $0$, the main challenge of the proof comes from establishing the required regularization strength $\lambda$, which depends on both $\phi$ (the quadratic form of target approximate eigenfunctions) and $\tilde{\phi}$ (the quadratic form of approximate eigenfunctions in the function class that we want to cover), and deriving the error bound~\eqref{eqn:9}.
The proof of Theorem~\ref{theorem:thm_eigenspace} can be found in Section~\ref{appendix:proof_eigenfunction}. 

\noindent\textbf{Relationship between 
	Theorem~\ref{theorem:thm_eigenspace} and Theorem~\ref{theorem:main_theorem_with_clustering}}.
Theorem~\ref{theorem:thm_eigenspace} is a strictly stronger version of Theorem~\ref{theorem:main_theorem_with_clustering} in the previous section.  
In the setting of Theorem~\ref{theorem:main_theorem_with_clustering}, the identity function of each minimal implementable cluster would be an achievable eigenfunction. 	
Theorem~\ref{theorem:thm_eigenspace} considers a more general situation than Theorem~\ref{theorem:main_theorem_with_clustering} where the minimal implementable clusters may not be well-defined, yet still we can show good results when the dimensionality $k$ is equal to the number of achievable eigenfunctions $m$.
Indeed, as we will see in Example~\ref{example1} in the next section, smaller representation dimensionality is needed to satisfy the Assumption~\ref{ass:5} of Theorem~\ref{theorem:thm_eigenspace} than to satisfy Assumption~\ref{assumption:implementability} of Theorem~\ref{theorem:main_theorem_with_clustering} (that is, the number of minimal implementable clusters in the graph).
We will mainly use Theorem~\ref{theorem:thm_eigenspace} for the examples because it’s more general and easier to be used, whereas we present Theorem~\ref{theorem:main_theorem_with_clustering} because it’s more intuitive to understand. 

\subsection{An end-to-end result}\label{section:end-to-end}
In this section, we also provide an end-to-end result where both the representation function and the downstream linear head are learned from the empirical dataset, and the analysis requires additional considering the complexity of the family class. 

Given an empirical unlabeled pre-training dataset that contains $\npre$ positive pairs $\{(x_1, x_1^+), (x_2, x_2^+), \cdots, (x_\npre, x_\npre^+)\}$, we define the empirical contrastive loss:
\begin{align}
	\eloss_{\lambda}(f) := \frac{1}{\npre}\sum_{i=1}^{\npre}\left(\norm{f(x_i)-f(x_i^+)}_2^2\right) + \lambda \cdot \norm{\sum_{i=1}^\npre[f(x_i)f(x_i)^\top] - \identity}_F^2.
\end{align}

We introduce the following notion of Rademacher complexity of the function class $\fclass$.
\begin{definition}[Rademacher complexity]
	We define the Rademacher complexity as  
	\begin{align}
		\rad{n}(\fclass) = \underset{{x_1\sim\pdata,  \cdots, x_n\sim\pdata} \atop {\eps\sim\{-1, 1\}^n}}{\Exp} \left[\sup_{f\in\fclass}\sup_{i\in[k]} \big| \frac{1}{n} \sum_{j=1}^n \epsilon_j  f(x_j)_i  \big|\right].
	\end{align}
\end{definition}

We have the following theorem which gives an error bound when both the contrastive representation and the downstream linear head are learned from finite samples.
\begin{theorem}\label{theorem:thm_eigenspace_end_to_end}
	Let $C_f$ be the upper bound of the feature norm, i.e., $\norm{f(x)}_2\le C_f$ for all $f\in\fclass$ and $x\in\data$. 
	Suppose function $\feig\in\fclass$ satisfies Assumptions~\ref{assumption:no_other_eigenvector_2} with $(\tilde{\phi}, \epsilon)$ and Assumption~\ref{assumption:downstream_2} with $(B, \zeta)$. 
	Suppose $\tilde{\phi}>2\phi$.
	Set $k=m$ and $\lambda\ge2\tilde{\phi}+1$. Given a random sample of $\npre$ positive pairs from the pretraining distribution, we learn a representation map $\hat{f} = \argmin_{f \in \fclass} \eloss_\lambda(f)$. Then, given $\nds$ labeled data from the downstream task $\{(x_1, \vec{y}(x_1)), (x_2, \vec{y}(x_2)), \cdots, (x_\nds, \vec{y}(x_\nds))\}$, we learn a linear head $\widehat{W}\in\real^{r\times k}$ such that
	\begin{align}
		\widehat{W} := \argmin_{W \text{ s.t. }\norm{W}_F\le 4B} \frac{1}{\nds} \sum_{i=1}^\nds\big\|{W\hat{f}(x_i)-\vec{y}(x_i)}\big\|_2^2.
	\end{align}
	Then, with probability at least $1-\delta$, we have
	\begin{align}
		\Exp_{x\sim\pdata}\big[\big\|{\widehat{W}\hat{f}(x)-\vec{y}(x)}\big\|_2^2\big] \lesssim & \zeta + B^2k\left(\epsilon+\frac{{\phi}}{\lambda}\right) \nonumber\\ &+ \frac{C_f^4k^2\lambda + B^2k}{\tilde{\phi}-2\phi}\left(\rad{\npre}(\fclass) + \sqrt{\frac{\log(k^2/\delta)}{\npre}}\right)+ BC_f \sqrt{\frac{\log(1/\delta)}{\nds}}.
	\end{align}
\end{theorem}

We note that the unsupervised sample complexity depends on the Rademacher complexity of the function class, whereas the supervised sample complexity doesn't.  The term $BC_f$ intuitively captures the complexity of the learned representation and usually increases as the feature dimensionality increases.
The fact that the downstream sample complexity scales with $BC_f$ demonstrates the benefit of learning a lower dimensional feature leveraging the inductive bias of the function class. The proof of Theorem~\ref{theorem:thm_eigenspace_end_to_end} is in Section~\ref{appendix:proof_end_to_end}.

\section{Instantiations on several synthetic data distributions}\label{section:examples}
In this section, we instantiate our previous theory on several examples of data distributions and show that when the model class has limited capacity, one can learn low-dimensional representations using contrastive learning and solve the downstream task with simple linear probing. In all of these examples, if we use a much more expressive model class, the representation dimensionality needs to be much higher, and hence more downstream samples are needed. These results demonstrate the benefit of leveraging inductive biases of the model architecture in contrastive learning.

\subsection{Linear functions}\label{section:example1}
Our first example is the hypercube example proposed in the work of \citet{saunshi2022understanding}. 
\begin{example}\label{example1}
	The natural data $\bar{x}\sim \{-1, 1\}^d$ is the uniform distribution over the $d$-dimensional cube. Given a natural data $\bar{x}$, an augmented data $x\sim\aug(\bar{x})$ is sampled as follows: first uniformly sample $d-s$ independent scalars $\tau_{s+1}\sim[\frac{1}{2}, 1], \cdots, \tau_{d}\sim[\frac{1}{2}, 1]$, then scale the $(s+1)$-th to $d$-th dimensions of $\bar{x}$ with $\tau_{s+1}, \cdots, \tau_{d}$ respectively, while keeping the first $s$ dimensions the same. 
	Intuitively, the last $d-s$ dimensions correspond to spurious features that can be changed by data augmentation, and the first $s$ dimensions are invariance features that contain information about the downstream task.
	The downstream task is a binary classification problem, where the label $y({x})=\sgn({x}_i)$ is the sign function of one of the first $s$ dimensions $i\in[s]$. 
\end{example}

We consider contrastive learning with the linear function class defined below:
\begin{definition}[Linear function class]
	Let $U\in\real^{k\times d}$ be a matrix and we use $f_{U}(x)=Ux$ to denote the linear function with weight matrix $U$. We define the $k$-dimensional linear function class as $\linearf = \{f_U: U\in\real^{k\times d}\}$.
\end{definition}

\citet{saunshi2022understanding} directly compute the learned representations from contrastive learning. 
We can instantiate Theorem~\ref{theorem:thm_eigenspace} to this example and get the following guarantees:
\begin{theorem}\label{theorem:example1}
	In Example~\ref{example1}, suppose we set the output dimensionality as $k=s$ and learn a linear representation map that minimizes the contrastive loss $\hat{f} = \argmin_{f\in\linearf} \loss_\lambda(f)$ for any $\lambda>0$. Then, there exists a linear head $w\in\real^{k}$ with zero downstream error, that is,
	\begin{align}
		\Exp_{x\sim\pdata} [(w^\top \hat{f}(x) - y(x))^2] = 0.
	\end{align}
	 In contrast, suppose the function class is the set of universal function approximators $\unif$. So long as the output dimensionality is no more than $2^{d-1}$, there exists solution $\hat{f}'\in\argmin_{f\in\unif} \loss_\lambda(f)$ such that any linear head $w\in\real^k$ has bad test error, that is
	\begin{align}
		\Exp_{x\sim\pdata}[(w^\top \hat{f}'(x) - y(x))^2] \ge 1.
	\end{align}
\end{theorem}

The proof of Theorem~\ref{theorem:example1} is in Section~\ref{appendix:example1}.

We note that as an implication of the lower bound, previous works that analyze universal function approximators~\citep{arora2019theoretical, tosh2021contrastive, haochen2021provable} wouldn't be able to show good downstream accuracy unless the representation dimensionality is larger than $2^{d-1}$. In contrast, our theory that incorporates the inductive biases of the function class manages to show that a much lower representation dimensionality $k=s$ suffices.

We also note that this example shows a situation where Theorem~\ref{theorem:thm_eigenspace} works but Theorem~\ref{theorem:main_theorem_with_clustering} doesn't, hence demonstrating how our theory derived from the eigenfunction viewpoint allows for lower representation dimensionality. 
There are $2^s$ model-restricted minimal clusters in the graph, each encoded by one configuration of the $s$ feature dimensions. Thus, applying Theorem~\ref{theorem:main_theorem_with_clustering} gives $k=2^s$ (which is worse than $k=s$ but better than $k=2^d$).
However, all the functions in $\linearf$ that implement some clusters span a $s$-dimensional subspace, thus we can find $s$ eigenfunctions that satisfy Assumption~\ref{assumption:no_other_eigenvector_2}.
As a result, learning $s$-dimensional representations already suffices for solving the downstream task. 

\subsection{ReLU networks}
In the previous example, the downstream task is only binary classification where the label is defined by one invariant feature dimension. 
Here we show that when we use a ReLU network as the model architecture, the linear probing can solve more diverse downstream tasks where the label can depend on the invariant feature dimensions arbitrarily.

\begin{example}\label{example2}
	The natural data distribution and the data augmentation are defined in the same way as Example~\ref{example1}. The downstream task is a $m$-way classification problem such that the label function  $y(\cdot):\data\rightarrow[m]$ satisfies $y({x}) = y({x}')$ if ${x}_{1:s} = {x}'_{1:s}$. In other words, the label only depends on the first $s$ dimensions of the data.
\end{example}
\begin{definition}[ReLU networks]
	Let $U\in\real^{k\times d}$ and $b\in\real^k$, we use $f_{U, b} = \sigma(Wx+b)$ to denote the ReLU network with weight $U$ and bias $b$, where $\sigma$ is the element-wise ReLU activation. We define the $k$-dimensional ReLU network function class as $\reluf = \{f_{U, b}: U\in\real^{k\times d}, b\in\real^k\}$.
\end{definition}

We have the following theorem which shows the effectiveness of the ReLU network architecture.
\begin{theorem}\label{theorem:example2}
	In Example~\ref{example2}, suppose we set the output dimensionality $k=2^s$ and learn a ReLU network representation map $\hat{f} = \argmin_{f\in\reluf} \loss_\lambda(f)$ for some $\lambda>0$. Then, we can find a linear head $W\in\real^{r\times k}$ with zero downstream error, that is,
	\begin{align}\label{eqn:6}
		\Exp_{x\sim\pdata}\big[\big\|W\hat{f}(x) - e_{y(x)}\big\|_2^2\big]=0.
	\end{align}
	
	In contrast, suppose the function class is the set of universal function approximators $\unif$. So long as the output dimensionality is no more than $2^{d-s}$, there exists solution $\hat{f}'\in\argmin_{f\in\unif} \loss_\lambda(f)$ such that any linear head $W\in\real^{m\times k}$ has bad test error, that is, 
	\begin{align}
		\Exp_{x\sim\pdata}\big[\big\|{W \hat{f}'(x) - e_{y(x)}}\big\|_2^2\big] \ge \frac{1}{2}.
	\end{align}
\end{theorem}

When $d\gg s$, the above theorem shows that constraining the feature map to the ReLU networks can achieve good downstream performance with much smaller feature dimensionality. The proof of Theorem~\ref{theorem:example2} is in Section~\ref{appendix:example2}.

\subsection{Lipschitz continuous functions}
In many real-world settings where a neural network is trained with weight decay regularization and stochastic gradient descent, the resulting model usually has a limited weight norm and a relatively small Lipschitz constant~\citep{jiang2019fantastic} partly because stochastic gradient descent implicit prefers Lipschitz models, which tend to generalize better~\citep{damian2021label,wei2020improved,wei2019data,li2021happens,foret2020sharpness}. 
Here we provide an example showing that restricting the model class to Lipschitz continuous functions allows us to use lower dimensional representations. In particular, we consider the following example where a large number of clusters are located close to each other despite being disconnected in the positive-pair graph. Our result shows that contrastive learning with Lipschitz continuous functions would group those clusters together, allowing for lower representation dimensionality.

\begin{example}\label{example3}
	Let $S_1, S_2, \cdots, S_\ncluster \subset\real^d$ be $\ncluster$ manifolds in $\real^d$,  each of which contains lots of disconnected subsets. 
	Suppose the diameter of every manifold is no larger than $\rho$, that is for any $i\in[\ncluster]$ and two data $x, x' \in S_i$, we have $\norm{x-x'}_2\le\rho$. We also assume that different manifolds are separated by $\gamma$, that is for any $i,j\in[\ncluster]$ such that $i\ne j$, and $x\in S_i$, $x'\in S_j$, we have $\norm{x-x'}_2\ge \gamma$.  
	The data distribution $\pdata$ is supported on $S_1\cup S_2 \cup\cdots\cup S_\ncluster$ and satisfies $\Pr_{x\sim\pdata}(x\in S_i)={1}/{\ncluster}$ for every $i\in[\ncluster]$. A positive pair only contains data from the same $S_i$ (but not all pairs of datapoints from the same set need to be positive pairs). The downstream task is a $m$-way classification problem such that the label function $y(\cdot): \data\rightarrow[m]$ satisfies $y({x}) = y({x}')$ if ${x}$ and ${x'}$ belong to the same set $S_i$.
\end{example}

We introduce the following family of Lipschitz continuous functions with parameter $\kappa$:
\begin{definition}[$\kappa$-Lipschitz continuous functions]
	A function $f: \real^d\rightarrow\real^k$ is $\kappa$-Lipschitz if $\norm{f(x)-f(x')}_2\le\kappa\norm{x-x'}_2$ for all $x, x'\in\real^d$. We define the $\kappa$-Lipschtiz function class
	$\lipf$ as the set of all $\kappa$-Lipschitz continuous functions in $\real^d\rightarrow\real^k$.
\end{definition}

We have the following theorem:
\begin{theorem}\label{theorem:example3}
	In Example~\ref{example3}, suppose $\kappa\ge {\sqrt{2\ncluster}}/{\gamma}$. Let the output dimensionality $k=\ncluster$ and learn a $\kappa$-Lipschitz continuous function $\hat{f} \in\argmin_{\lipf}\loss_{\lambda}(f)$ for some $\lambda>0$. Then, we can find a linear head $W\in\real^{m\times k}$ with small downstream error, that is,
	\begin{align}\label{eqn:10}
		\Exp_{x\sim\pdata}\big[\big\|{W\hat{f}(x) - e_{y(x)}}\big\|_2^2\big]\le 2rm \kappa^2\rho^2.
	\end{align}
	On the other hand, suppose the positive-pair graph contains $\nclustertrue$ disconnected clusters, and the function class is the set of universal function approximators $\unif$. So long as the output dimensionality $k\le\frac{1}{2} \cdot \nclustertrue$, there exists solution $\hat{f}'\in\argmin_{f\in\unif} \loss_\lambda(f)$ such that any linear head $W\in\real^{m\times k}$ has bad test error, that is,
	\begin{align}
		\Exp_{x\sim\pdata}\big[\big\|{W \hat{f}'(x) - e_{y(x)}}\big\|_2^2\big] \ge \frac{1}{2}.
	\end{align}
\end{theorem}
We note that a smaller $\kappa$ (hence smoother function class) decreases the RHS of \eqref{eqn:10} and leads to better downstream performance. When $r_0$ is much bigger than $r$, the above theorem shows that constraining the feature map to the Lipschitz continuous functions can allow for much smaller feature dimensionality. The proof of Theorem~\ref{theorem:example3} is in Section~\ref{appendix:example3}.

\subsection{Convolutional neural networks}
Our last example shows that convolutional neural networks can learn contrastive representation more efficiently than fully connected networks when the downstream task has a certain translational invariance structure. We consider the following data generative model where the data contains a patch that determines the downstream label.
\begin{example}\label{example4}
	The natural data $\bar{x}\in\real^d$ is defined as follows: for some location $t\in[d]$ (which could depend on $\bar{x}$) and consecutive $s$ dimensions $\bar{x}_{t:t+s-1}$ (the informative patch), we have $\bar{x}_{t:t+s-1}\in\{-\gamma, \gamma\}^s$ where $\gamma>1$.\footnote{Here we denote $\bar{x}_{d+i}=\bar{x}_i$.}
	The other $d-s$ dimensions of $\bar{x}$ (spurious dimensions) are all in $\{-1, 1\}$. Given a natural data $\bar{x}$, its augmentations are generated by the following procedure: for every spurious dimension $i$, first sample $\tau_i\sim Uni[0, 1]$, then multiply $\bar{x}_i$ by $\tau_i$. The augmentation keeps the informative patch the same as original. The downstream task  is a $m$-way classification problem such that the label function $y(\cdot):\data\rightarrow[m]$ satisfies $y({x}) = y({x}')$ if the informative patches for $x$ and $x'$ are the same.
\end{example}

We consider the following convolutional neural network model with $k$ channels.
\begin{definition}[Convolutional neural networks]
	Let $U = [u_1, u_2, \cdots, u_k]^\top \in\real^{k\times s}$ and $b\in\real^k$. We use $f^{\text{conv}}_{U, b}:\data\rightarrow \real^k$ to represent the following convolutional neural network: $f^{\text{conv}}_{U, b}(x)_i = \sum_{t=1}^d \sigma(u_i^\top x_{t:t+s-1}+b_i)$  where $\sigma$ is ReLU activation function,  and $f^{\text{conv}}_{U, b}(x) = [f^{\text{conv}}_{U, b}(x)_1, \cdots, f^{\text{conv}}_{U, b}(x)_k]^\top$. We define the convolutional neural network class $\convf = \{f^{\text{conv}}_{U, b}: U\in\real^{k\times s}, b\in\real^k\}$.
\end{definition}

Ideally, we would like a learn a feature map where the ReLU function is only activated on the informative patch.
We have the following theorem which shows that contrastive learning with convolutional neural networks can indeed learn such ideal feature maps, hence allowing lower representation dimensionality than using fully-connected ReLU networks.
\begin{theorem}\label{theorem:example4}
	In Example~\ref{example4}, let output dimensionality $k=2^s$ and learn a convolutional neural network $\hat{f} \in\argmin_{\convf}\loss_{\lambda}(f)$ for some $\lambda>0$. Then, we can find a linear head $W\in\real^{m\times k}$ with zero downstream error, that is,
	\begin{align}
		\Exp_{x\sim\pdata}\big[\big\|{W\hat{f}(x) - e_{y(x)}}\big\|_2^2\big]=0.
	\end{align}
	
	On the other hand, suppose the function class is the set of ReLU networks $\reluf$, so long as the output dimensionality is less than $d\times2^{s-1}$, there exists a function $\hat{f}'\in\argmin_{f\in\reluf}\loss_{\lambda}(f)$ such that any linear head $W\in\real^{m\times k}$ has bad test error, that is, 
	\begin{align}
		\Exp_{x\sim\pdata}\big[\big\|{W \hat{f}'(x) - e_{y(x)}}\big\|_2^2\big] \ge \frac{1}{2}.
	\end{align}
\end{theorem}
When $d\gg2$, the above theorem shows that constraining the feature map to convolutional networks can allow for much smaller feature dimensionality.
The proof of Theorem~\ref{theorem:example4} is in Section~\ref{appendix:example4}.

\section{Simulations}\label{section:experiments}
In this section, we empirically validate the practical relevance of our key assumptions that the model architecture (1) can implement clusters that align with downstream tasks, and (2) cannot break the data into \emph{too many} well-separated clusters. 

Recall that we have two versions of the model-implementable clustering assumptions, Assumption~\ref{assumption:connectivity} and Assumption~\ref{ass:5}, where the latter can be viewed as a soft-clustering relaxation of the former as argued in Section~\ref{section:eigenspace}. Since Assumption~\ref{ass:5} is weaker and also easier to evaluate empirically, we design algorithms to (heuristically) compute the following value $\phi_m$ defined in the last sentence of Assumption~\ref{ass:5} (where $\feig$ is $m$-dimensional):
\begin{align}
	\phi_m = \Exp_{(x, x^+)\sim\ppos} \big[\big\|{\feig(x)-\feig(x^+)}\big\|_2^2\big]  \quad \text{ s.t. } \quad \Exp_{x\sim\pdata}\big[\feig(x)\feig(x)^\top\big] = \identity.
\end{align}
 Note that a small $\phi_m$ indicates that there exists $m$ eigenfunctions within the function class $\fclass$ with small eigenvalues, or intuitively $m$-way model-implementable partitions of the graph with small inter-cluster connections. In other words, $\phi_m$ is a surrogate for how the architecture can partition the graph into $m$ clusters. 

We propose a method that approximately computes the value of $\phi_m$. We empirically compute $\phi_m$ by first minimizing the contrastive loss $$\loss_{\lambda}(f_\theta) = \Exp_{(x, x^+)\sim \ppos}[\norm{f(x)-f(x^+)}_2^2] + \lambda \cdot  \norm{\Exp_{x\sim\pdata}[f(x)f(x)^\top] - \identity}_F^2$$ with representation dimension $k=m$ and a heavily-tuned regularization strength $\lambda \in \{0.1, 0.3, 1, 3, 10, 30, 100, 300, 1000\}$.
Then, we whiten the obtained model $f_{\hat{\theta}}(x)$ to have exactly the covariance $\identity/r$, that is, $\bar{f}(x) = \Exp_{x\sim\pdata}[f_{\hat{\theta}}(x)f_{\hat{\theta}}(x)^\top]^{-\frac{1}{2}}f_{\hat{\theta}}(x)/\sqrt{m}$ which is a valid solution for \eqref{eqn:7} and \eqref{eqn:8}. We compute $\Exp_{(x, x+)\sim \ppos}[\norm{\bar{f}(x)-\bar{f}(x^+)}_2^2] $ under various choices of $\lambda$ and pick the smallest result as the final value of the estimated $\phi_m$. 

It's more challenging to verify the conditions saying the clusters are not breakable into smaller ones, such as Assumption~\ref{assumption:separability} or Assumption~\ref{assumption:no_other_eigenvector_2}. As a crude surrogate, here we evaluate $\phi_m$ for larger $m$ and show that $\phi_m$ will be larger when $m$ becomes larger. For example, we will see that $\phi_{10}$ for CIFAR-10 is relatively small for all architectures, indicating that the model architectures can implement a 10-way partition with small inter-cluster connections. On the other hand, $\phi_{300}$ becomes significantly larger which indicates that there are no ways for the model to implement a good 300-way partition. 

We run experiments on CIFAR-10 and compute $\phi_m$ for $m\in\{10, 30, 100, 300\}$ for four models: ResNet-18, ResNet-101, Wide ResNet (WRN) and vision transformer (ViT). We list the results in the table below. Here all experiments are run for 200 epochs using SGD with a starting learning rate $0.01$ and a cosine learning rate schedule. More details can be found in Section~\ref{appendix:experiments}.
\begin{table}[!htb]
	\begin{center}
			{\renewcommand{\arraystretch}{1.1}
					\begin{tabular}{|l|l l l l|}
			\hline
							 & $m=10$ & $m=30$  & $m=100$ & $m=300$\\
			\hline
			ResNet-18 & 0.131 & 0.166 & 0.224 & 0.521\\
			ResNet-101 & 0.053 & 0.090 & 0.163 & 0.459\\
			WRN & 0.080 & 0.093 & 0.138 & 0.340\\
			ViT & 0.072 & 0.108 & 0.168 & 0.389\\
			\hline
					\end{tabular}}
		\end{center}
	\caption{The empirical $m$-way separability $\phi_m$ on CIFAR-10 dataset. \vspace{10pt}
		}
\end{table}

We note that for any fixed model, $\phi_m$ increases as $m$ increases from $10$ to $300$, suggesting that although the network can partition the data relatively well into $10$ clusters, it cannot partition the data into $100$ well-separated clusters, which supports our theoretical assumptions. 
More concretely, if one wants to find $10$ model-implementable clusters, then the cut value (the fraction of inter-cluster edges) will be $0.053$ for ResNet-101, whereas the best $300$-way model-implementable clustering has $0.459$ cut value. This suggests that there is not any good $300$-way clustering that is implementable by ResNet-101.

Moreover, we found that indeed ResNet-101, WideResNet, and ViT can implement more separate clusters (clusters with smaller number of edges in between) than ResNet-18. This suggests that our assumptions can distinguish different model architectures. 
\vspace{-3mm}
\section{Conclusion}
In this paper, we provide a theoretical analysis of contrastive learning that incoporates the inductive biases of the model class. We prove that contrastive learning with appropriate model architectures allows for lower representation dimensionality (hence better sample complexity), and instantiate this theory on several interesting examples. One open question is to allow for overparameterization, i.e., showing similar results when $k>m$. Under our current theoretical framework, the optimal features are the top eigenfunctions of the graph Laplacian, and the learned representation may omit some top eigenfunctions when $k>m$ and get suboptimal downstream performance. To allow for a more flexible choice of $k$, we believe that additional assumptions on the structure of the function class need to be made. Another open question would be theoretically studying the role of inductive biases under more common contrasive losses (e.g., the InfoNCE loss~\citep{oord2018representation}). 
Finally, we note that our work only concerns the inductive biases originating from the model architecture, whereas in practice the learned representations may also depend on the optimization method~\cite{Liu2022SamePL}. Hence, one interesting future direction would be studying how the \emph{implicit bias} introduced by the optimizer influences contrastive learning.

\section*{Acknowledge}
Toyota Research Institute provided funds to support this work.

\bibliographystyle{plainnat}
\bibliography{all}

\newpage
\appendix
\section{Addition experimental details}\label{appendix:experiments}
We train a ResNet-18/ResNet-101/Wide-ResNet/Vit model on CIFAR-10 and test the $\phi_m$ on the test set. We train with SGD using initial learning rate $0.01$ and with a cosine decaying schedule, and use a momentum $0.9$, weight decay $0.0005$ and a batch size of $128$. All experiments are run for 200 epochs. 
We test with $m\in\{10, 30, 100, 300\}$ and grid search using $\lambda\in\{0.1, 0.3, 1, 3,10, 30, 100, 300, 1000\}$, and for each $m$ we report the minimal $\phi_m$ that is achieved by any $\lambda$. 

In our experiments, we always initialized the model with a pre-trained model. Our pre-trained model is the result of running the spectral contrastive learning algorithm~\cite{haochen2021provable} for $800$ epochs on CIFAR-10, following the same hyperparameter choice as in their paper.

\section{Proofs for Section~\ref{section:main_theory}}\label{appendix:proof_cluster}

\begin{proof}[Proof of Theorem~\ref{theorem:main_theorem_with_clustering}]
	
	Let $P_i := \Pr_{x\sim\pdata}(x\in S_i)$ be the probability of $S_i$. Let $f^*$ be the function $f^*(x) = \frac{1}{\sqrt{P_{\id_x}}} e_{\id_x}$. From Assumption~\ref{assumption:implementability} and Assumption~\ref{assumption:closeness_under_scaling}, we have $f^*\in\fclass$. 
	
	We first show that $f^*$ achieve small contrastive loss. For the regularizer term, we have
	\begin{align}\label{eqn:1}
		\Exp_{x\sim\pdata}\left[f^*(x)f^*(x)^\top\right] = \sum_{i\in[m]}  P_i \cdot \frac{1}{P_i} \cdot e_{\id_x}e_{\id_x}^\top = \identity.
	\end{align}
	Thus, we have $R(f^*) = 0$. For the discrepancy term, let $\pmin := \min_{i\in[m]} P_i$ be the probability mass of the smallest set, we have
	\begin{align}\label{eqn:2}
		\Exp_{x, x^+}\left[\norm{f^*(x) - f^*(x^+)}_2^2\right]\le \frac{1}{\pmin}\cdot \Pr_{(x, x+)\sim \ppos}(\id_x\ne\id_{x^+}) \le \frac{\alpha}{\pmin}.
	\end{align}
	Combining \eqref{eqn:1} and \eqref{eqn:2} we have 
	\begin{align}
		\loss_{\lambda}(f^*) = \Exp_{x, x^+}\left[\norm{f^*(x) - f^*(x^+)}_2^2\right]+\lambda\cdot R(f) \le \frac{\alpha}{\pmin}.
	\end{align}
	
	Since $\hat{f} = \argmin_{f \in F} \loss_\lambda(f)$ is the minimizer of contrastive loss within the function class, we have
	\begin{align}
		\loss_{\lambda}(\hat{f}) \le \loss_{\lambda}(f^*) \le \frac{\alpha}{\pmin}.
	\end{align}
	
	Define matrix 
	\begin{align}
		M :=\Exp_{x\sim\pdata}\left[\hat{f}(x)\hat{f}(x)^\top\right].
	\end{align}
	
	We have 
	\begin{align}
		\norm{M-\identity}_F^2 \le \frac{\loss_{\lambda}(\hat{f})}{\lambda} \le\frac{\alpha}{\lambda\pmin}.
	\end{align}
	
	Since $\lambda > \frac{\alpha}{\pmin}$, we know that $M$ is a full rank matrix, thus we can define function 
	\begin{align}
		\tilde{f}(x) := M^{-\frac{1}{2}} \hat{f}(x).
	\end{align}
	
	Let 
	\begin{align}
		Q := \Exp_{x\sim\pdata} \left[\tilde{f}(x)f^*(x)^\top\right],
	\end{align}
	and 
	\begin{align}
		\pif(x) := \tilde{f}(x) - Qf^*(x).
	\end{align}
	
	We know that 
	\begin{align}
		\Exp_{x\sim \pdata}\left[\pif(x)f^*(x)^\top\right]  = \Exp_{x\sim\pdata}\left[\tilde{f}(x)f^*(x)^\top\right] - Q\Exp_{x\sim\pdata}\left[f^*(x)f^*(x)^\top\right] = 0.
	\end{align}
	
	Using Assumption~\ref{assumption:separability} we have:
	\begin{align}\label{eqn:3}
		&\Exp_{(x, x+)\sim \ppos}\left[\norm{\pif(x)-\pif(x^+)}_2^2\right]\notag\\
		\ge&  \sum_{i\in[m]}(P_i-\alpha)\cdot \Exp_{(x, x+)\sim {\ppos}_i}\left[\norm{\pif(x)-\pif(x^+)}_2^2\right]\notag\\
		\ge& \beta \cdot\sum_{i\in[m]} (P_i-\alpha)\cdot \Exp_{x\sim{\pdata}_i, x'\sim{\pdata}_i}\left[\norm{\pif(x)-\pif(x')}_2^2\right]\notag\\
		=&2 \beta\cdot \sum_{i\in[m]}(P_i-\alpha)\cdot \Exp_{x\sim{\pdata}_i}\left[\norm{\pif(x)}_2^2\right]\notag\\
		=&2\beta \cdot (1 - \frac{\alpha}{\pmin})\cdot\Exp_{x\sim{\pdata}}\left[\norm{\pif(x)}_2^2\right].
	\end{align}
	
	On the other hand, we have
	\begin{align}\label{eqn:4}
		&\Exp_{(x, x+)\sim \ppos}\left[\norm{\pif(x)-\pif(x^+)}_2^2\right] \notag\\
		\le& \Exp_{(x, x+)\sim \ppos}\left[\norm{\tilde{f}(x)-\tilde{f}(x^+)}_2^2\right]\notag\\
		\le& \norm{M^{-1}}_{\text{spec}}\cdot \Exp_{(x, x+)\sim \ppos}\left[\norm{\hat{f}(x)-\hat{f}(x^+)}_2^2\right]\notag\\
		\le& \left(1+\sqrt{\frac{\alpha}{\lambda\pmin}}\right) \cdot \frac{\alpha}{\pmin}.
	\end{align}
	
	Combining \eqref{eqn:3} and \eqref{eqn:4} we have
	\begin{align}
		\Exp_{x\sim{\pdata}}\left[\norm{\pif(x)}_2^2\right] \le \left(1+\sqrt{\frac{\alpha}{\lambda\pmin}}\right) \cdot \frac{\alpha}{2\beta(\pmin-\alpha)}.
	\end{align}
	
	By Lemma~\ref{lemma:projection_other_direction}, we know that there exists a matrix $U\in\real^{m\times m}$ such that
	\begin{align}
		\Exp_{x\sim\pdata}\left[\norm{f^*(x)-UM^{-1/2}\hat{f}(x)}_2^2\right] &\le \left(1+\sqrt{\frac{\alpha}{\lambda\pmin}}\right) \cdot \frac{\alpha}{2\beta(\pmin-\alpha)}\\
		&\le \frac{\alpha}{\beta(\pmin-\alpha)}.
	\end{align}
	Thus, if we define matrix $W = \text{diag}\{\sqrt{P_1}, \sqrt{P_2}, \cdots, \sqrt{P_m}\} U M^{-1/2}$, then we have
	\begin{align}
		\Exp_{x\sim\pdata}\left[\norm{e_{\id_x}-W\hat{f}(x)}_2^2\right] &\le \pmax \Exp_{x\sim\pdata}\left[\norm{f^*(x)-UM^{-1/2}\hat{f}(x)}_2^2\right] \\
		&\le \pmax \frac{\alpha}{\beta(\pmin-\alpha)},
	\end{align}
	which finishes the proof.
	
\end{proof}

\begin{lemma}\label{lemma:projection_other_direction}
	Suppose $f:\data\rightarrow\real^m$ and $g:\data\rightarrow\real^m$ are two functions defined on $\data$ such that 
	\begin{align}
		\Exp_{x\sim\pdata}\left[f(x)f(x)^\top\right] = \Exp_{x\sim\pdata}\left[g(x)g(x)^\top\right] = \identity.
	\end{align}
	Define the projection of $f$ onto $g$'s orthogonal subspace as:
	\begin{align}
		\pif(x) = f(x)-\Exp_{x'\sim\pdata}\left[f(x)g(x)^\top\right] g(x).
	\end{align}
	Then, there exists  matrix $U\in\real^{m\times m}$ such that
	\begin{align}
		\Exp_{x\sim\pdata}\left[\norm{g(x)-Uf(x)}_2^2\right] = \Exp_{x\sim\pdata}\left[\norm{\pif(x)}_2^2\right].
	\end{align}
\end{lemma}

\begin{proof}[Proof of Lemma~\ref{lemma:projection_other_direction}]
	Let matrix 
	\begin{align}
		U = \Exp_{x'\sim\pdata}\left[g(x)f(x)^\top\right].
	\end{align}
	We have
	\begin{align}
		&\Exp_{x\sim\pdata}\left[\norm{g(x)-Uf(x)}_2^2\right] \\
		=& \Exp_{x\sim\pdata}\left[\norm{g(x)}_2^2\right] - 2\Exp_{x\sim\pdata}\left[g(x)^\top Uf(x)\right] + \Exp_{x\sim\pdata}\left[f(x)^\top U^\top U f(x)\right]\\
		=&m - 2\norm{U}_F^2 + \norm{U}_F^2\\
		=& m-\norm{U}_F^2.
	\end{align}
	On the other hand, we have
	\begin{align}
		&\Exp_{x\sim\pdata} \left[\norm{\pif(x)}_2^2\right] \\
		=& \Exp_{x\sim\pdata}\left[\norm{f(x)- U^\top g(x)}_2^2\right]\\
		=& \Exp_{x\sim\pdata}\left[\norm{f(x)}_2^2\right] -2\Exp_{x\sim\pdata}\left[f(x)^\top U^\top g(x)\right] + \Exp_{x\sim\pdata}\left[g(x)^\top UU^\top g(x) \right]\\
		=& m-\norm{U}_F^2.
	\end{align}
	Thus, we have 
	\begin{align}
		\Exp_{x\sim\pdata}\left[\norm{g(x)-Uf(x)}_2^2\right] = \Exp_{x\sim\pdata}\left[\norm{\pif(x)}_2^2\right],
	\end{align}
	which finishes the proof.
\end{proof}

\section{Proofs for Section~\ref{section:eigenspace}}\label{appendix:proof_eigenfunction}
\begin{proof}[Proof of Proposition~\ref{proposition:eigenvector}]
	Define function
	\begin{align}
		\tilde{g}(x) = \sqrt{\pdata(x)}g(x).
	\end{align}
	Define the symmetric Laplacian operator \begin{align}
		{\tlaplacian}(\tilde{g})(x) = \tilde{g}(x) - \int\frac{\ppos(x, x')}{\sqrt{\pdata(x)}\sqrt{\pdata(x')}}\tilde{g}(x')dx'.
	\end{align}
	It can be verified that 
	\begin{align}
		\int_x \tilde{g}(x) {\tlaplacian}(\tilde{g})(x) = 0.
	\end{align}
	Notice that the operator $\tlaplacian$ is PSD, we have that 
	\begin{align}
		\int_x  ({\tlaplacian}(\tilde{g})(x))^2 = 0,
	\end{align}
	which is equivalent to 
	\begin{align}
		\Exp_{x\sim\pdata}  \left[({\laplacian}(\tilde{g})(x))^2 \right]= 0,
	\end{align}
	hence finishes the proof.
\end{proof}

\begin{proof}[Proof of Theorem~\ref{theorem:thm_eigenspace}]
	Notice that $\loss_{\lambda}(\feig) \le\phi$, we know that $ \loss_{\lambda}(\hat{f}) \le\phi$, so 
	\begin{align}
		\norm{\Exp_{x\sim\pdata} \left[\hat{f}(x)\hat{f}(x)^\top\right] -\identity}_F^2 \le \frac{\phi}{\lambda},
	\end{align}
and 
\begin{align}
	\Exp_{(x, x^+)\sim\ppos} \left[\norm{\hat{f}(x) - \hat{f}(x^+)}_2^2\right] \le \phi.
\end{align}

Therefore, for any $i\in[k]$, we have
\begin{align}
	\Exp_{x\sim\pdata}\left[\hat{f}(x)_i^2\right] \ge 1-\sqrt{\frac{\phi}{\lambda}},
\end{align}
and 
\begin{align}
	\Exp_{(x, x^+)\sim\ppos} \left[\left(\hat{f}(x)_i - \hat{f}(x^+)_i\right)^2\right] \le \phi.
\end{align}

Thus, 
\begin{align}
	\Exp_{(x, x^+)\sim\ppos} \left[\left(\hat{f}(x)_i - \hat{f}(x^+)_i\right)^2\right] \le  \frac{\phi}{1-\sqrt{\frac{\phi}{\lambda}}} \cdot \Exp_{x\sim\pdata}\left[\hat{f}(x)_i^2\right] \le \tilde{\phi} \cdot \Exp_{x\sim\pdata}\left[\hat{f}(x)_i^2\right].
\end{align}
	
	
	In Assumption~\ref{assumption:no_other_eigenvector_2}, set $g=\hat{f}_i$ and sum over $i=1, 2, \cdots, k$, we have that for some matrix $\tilde{W}$,
	\begin{align}
		\Exp_{x\sim\pdata}\left[\norm{\tilde{W}\feig(x) - \hat{f}(x)}_2^2\right]\le k\epsilon.
	\end{align}
	
	Let matrix $Q := \Exp_{x\sim\pdata}[\hat{f}(x)\hat{f}(x)^\top]$, we have that
	\begin{align}
		\Exp_{x\sim\pdata}\left[\norm{Q^{-1/2}\hat{f}(x) - \hat{f}(x)}_2^2\right] \le \frac{2\phi}{\lambda}\cdot \Exp_{x\sim\pdata}\left[\norm{ \hat{f}(x)}_2^2 \right] \le \frac{2\phi}{\lambda} k \left(1+\sqrt{\frac{\phi}{\lambda}}\right).
	\end{align}
	Thus, 
	\begin{align}\label{eqn:5}
	\Exp_{x\sim\pdata}\left[\norm{\tilde{W}\feig(x) -Q^{-1/2} \hat{f}(x)}_2^2\right]\le 2k\epsilon + \frac{4\phi}{\lambda}k \left(1+\sqrt{\frac{\phi}{\lambda}}\right).
	\end{align}
		
	Define matrix
	\begin{align}
		M := \Exp_{x\sim\pdata}\left[\feig(x)\hat{f}(x)^\top Q^{-1/2}\right]
	\end{align}
	Using Lemma~\ref{lemma:projection_other_direction} and \eqref{eqn:5} we have
\begin{align}
	\Exp_{x\sim\pdata}\left[\norm{\feig(x) - MQ^{-1/2}\hat{f}(x)}_2^2\right] \le 2k\epsilon + \frac{4\phi}{\lambda}k \left(1+\sqrt{\frac{\phi}{\lambda}}\right) \le  2k\epsilon + \frac{8\phi}{\lambda}k.
\end{align}
Thus, using Assumption~\ref{assumption:downstream_2}, we have
\begin{align}
	&\Exp_{x\sim\pdata}\left[\norm{{W^*}MQ^{-1/2}\hat{f}(x) - e_{y(x)}}_2^2\right] \\
	\le& 2\Exp_{x\sim\pdata}\left[\norm{{W^*}\feig(x) - e_{y(x)}}_2^2\right] + 2\Exp_{x\sim\pdata}\left[\norm{{W^*}MQ^{-1/2}\hat{f}(x) - {W^*}\feig(x)}_2^2\right]\\
	\le& 2\zeta+4B^2k\epsilon + \frac{16\phi}{\lambda}B^2k.
\end{align}
\end{proof}

\subsection{Proof for Section~\ref{section:end-to-end}}\label{appendix:proof_end_to_end}
\begin{proof}[Proof of Theorem~\ref{theorem:thm_eigenspace_end_to_end}]
	We first show that with high probability, $\loss_{\lambda}(\hat{f})$ is small. Consider the following decomposition
	\begin{align}
		\loss_{\lambda}(f) - \eloss_{\lambda}(f) = T_1 + T_2 + T_3,
	\end{align}
where
\begin{align}
	T_1 = \Exp\left[\norm{f(x) - f(x^+)}^2\right] - \eExp\left[\norm{f(x) - f(x^+)}^2\right],
\end{align}
\begin{align}
	T_2 = \lambda \Tr\left(\Exp\left[f(x)f(x)^\top\right]^2 - \eExp\left[f(x)f(x)^\top\right]^2\right),
\end{align}
\begin{align}
	T_3 = -2\lambda \left(\Tr\left(\Exp\left[f(x)f(x)^\top\right]\right) - \Tr\left(\eExp\left[f(x)f(x)^\top\right]\right) \right).
\end{align}
Here we use $\Exp$ to denote expectation over population distribution, and $\eExp$ to denote expectation over the empirical dataset. 
Since $\norm{f(x) - f(x^+)}^2\le 4C_f^2$, using the uniform convergence bound via Rademacher complexity and Talagrand’s Contraction Lemma~\cite{ledoux1991probability}, we have that with probability at least $1-\frac{\delta}{6}$, 
\begin{align}
	T_1 \le 8C_f \rad{\npre}(\fclass) + 4C_f^2 \sqrt{\frac{\log(6/\delta)}{2\npre}}.
\end{align}
Notice that
\begin{align}
	T_2 \le 2C_f^2\lambda \cdot \sum_{i, j}\left| \Exp\left[f(x)_if(x)_j\right]  - \eExp\left[f(x)_if(x)_j\right]\right|,
\end{align}
we have that with probability at least $1-\frac{\delta}{6}$, 
\begin{align}
	T_2 \le 2C_f^2\lambda k^2\cdot C_f\rad{\npre}(\fclass) + 2C_f^2\lambda k^2\cdot 2C_f^2\sqrt{\frac{\log(6k^2/\delta)}{2\npre}}.
\end{align}
Notice that
\begin{align}
	T_3 \le 2\lambda\cdot  \left| \Exp\left[\norm{f(x)_2^2}\right] - \eExp\left[\norm{f(x)^2}\right] \right|,
\end{align}
we have that with probability at least $1-\frac{\delta}{6}$, 
\begin{align}
	T_3\le 2\lambda\cdot C_f\rad{\npre}(\fclass) + 2\lambda \cdot C_f^2\sqrt{\frac{\log(6/\delta)}{2\npre}}.
\end{align}
Thus, with probability at least $1-\frac{\delta}{2}$, we have the following inequality for all $f\in\fclass$:
\begin{align}
	\loss_{\lambda}(f) - \eloss_{\lambda}(f) \le (8C_f + 2C_f^3\lambda k^2 + 2C_f\lambda)\cdot\rad{\npre}(\fclass) + (4C_f^2 + 4C_f^4\lambda k^2 + 2C_f^2\lambda)\cdot \sqrt{\frac{\log(6k^2/\delta)}{2\npre}}.\label{eqn:31}
\end{align}
When the RHS of \eqref{eqn:31} (which we refer as ``RHS'' below) is smaller than $\frac{1}{2}\tilde{\phi} - \phi$, we have that
\begin{align}
	\loss_{\lambda}(\hat{f})\le \phi + RHS \le \frac{1}{2}\tilde{\phi}.
\end{align}
Thus, following the same proof as Theorem~\ref{theorem:thm_eigenspace}, we know that there exists a linear head $W\in\real^{r\times k}$ defined as
\begin{align}
	W = W^* \cdot \Exp_{x\sim\pdata}\left[\feig(x)\hat{f}(x)^\top\right] \cdot \Exp_{x\sim\pdata}[\hat{f}(x)\hat{f}(x)^\top]^{-1},
\end{align}
such that
\begin{align}
	\Exp_{x\sim\pdata}\big[\big\|{W\hat{f}(x)-\vec{y}(x)}\big\|_2^2\big] \le 2\zeta+4B^2k\epsilon + \frac{4(\phi+RHS)}{\lambda}B^2k.
\end{align}

Since $\lambda>2\tilde{\phi}$, we have $\norm{\Exp_{x\sim\pdata}[\hat{f}(x)\hat{f}(x)^\top]^{-1}}_2\le 2$ and $\norm{\left[\feig(x)\hat{f}(x)^\top\right]}_2\le 2$, thus $\norm{W}_F\le 4B$. Using uniform convergence bound again, we have that with probability at least $1-\frac{\delta}{2}$, 
\begin{align}\label{eqn:32}
	\Exp_{x\sim\pdata}\big[\big\|{\widehat{W}\hat{f}(x)-\vec{y}(x)}\big\|_2^2\big] \lesssim \zeta + B^2k(\epsilon+\frac{{\phi}+RHS}{\lambda}) + BC_f\cdot \sqrt{\frac{\log(1/\delta)}{\nds}}.
\end{align}

Since $\widehat{W}=0$ gives trivial test loss $1$, we can add an additional term to \eqref{eqn:32} to also cover the case when the RHS of \eqref{eqn:31} is larger than $\frac{1}{2}\tilde{\phi} - \phi$, as follows:
\begin{align}
	&\Exp_{x\sim\pdata}\big[\big\|{\widehat{W}\hat{f}(x)-\vec{y}(x)}\big\|_2^2\big] \\ \lesssim& \zeta + B^2k(\epsilon+\frac{{\phi}+RHS}{\lambda})  + \frac{C_f^4k^2\lambda}{\tilde{\phi}-2\phi}\left(\rad{\npre}(\fclass) + \sqrt{\frac{\log(k^2/\delta)}{\npre}}\right)+ BC_f \sqrt{\frac{\log(1/\delta)}{\nds}}\\
	\le& \zeta + B^2k(\epsilon+\frac{{\phi}}{\lambda})  + \frac{C_f^4k^2\lambda + B^2k}{\tilde{\phi}-2\phi}\left(\rad{\npre}(\fclass) + \sqrt{\frac{\log(k^2/\delta)}{\npre}}\right)+ BC_f \sqrt{\frac{\log(1/\delta)}{\nds}},
\end{align}
where the second inequality uses $\lambda > 2\tilde{\phi}>\tilde{\phi} - 2\phi$.

\end{proof}

%
%
%
%

\section{Proofs for Section~\ref{section:examples}}\label{appendix:proof_example}

\subsection{Proof for Example~\ref{example1}}\label{appendix:example1}

\begin{proof}[Proof of Theorem~\ref{theorem:example1}]
	Define $\hat{U} = \left[e_1, e_2, \cdots, e_s\right]^\top \in \real^{s\times d}$. We can verify that
	\begin{align}
		\Exp_{(x, x^+)\sim\ppos} \left[\norm{f_{\hat{U}}(x)-f_{\hat{U}}(x^+)}_2^2\right] = 0
	\end{align}
	and 
	\begin{align}
		\Exp_{x\sim\pdata}\left[f_{\hat{U}}(x)f_{\hat{U}}(x)^\top\right] = \identity.
	\end{align}
	Thus, we can view $f_{\hat{U}}$ as the $\feig$ in Section~\ref{section:eigenspace}.
	
	Let $U\in\real^{k\times d}$ and $i\in[k]$ such that 
	\begin{align}
		\Exp_{(x, x^+)\sim\ppos}[(f_U(x)_i  -  f_U(x^+)_i)^2] = 0.
	\end{align}
	Notice that $x$ and $x^+$ only differs on the $s+1$-th to $d$-th dimensions, we know that $U_i$ is $0$ on the $s+1$-th to $d$-th dimensions. Thus, we have that $U_i$ is in the span of $e_1, e_2, \cdots, e_s$, and as a result Assumption~\ref{assumption:no_other_eigenvector_2} holds wiht $\epsilon=0$.
	
	Since the downstream task's label is equal to $x_i$ for $i\in[s]$, we can set $W^* = e_i^\top$ and we would have
	\begin{align}
		\Exp_{x\sim\pdata}\left[(W^*f_{\hat{U}}(x) - \vec{y}(x))^2\right] = 0.
	\end{align}
	Hence Assumption~\ref{assumption:downstream_2} holds with $\alpha=0$ and $B=1$.
	
	Applying Theorem~\ref{theorem:thm_eigenspace} finishes the proof for the linear function class case.
	
	For the case of universal function approximators, without loss of generality we assume the downstream task's label only depends on the first dimension of $x$, i.e., $y(x) = \sgn(x_1)$. When $k\le 2^{d-1}$, we can construct a function $f:\data\rightarrow{\real^{k}}$ such that for every diemnsion $j\in[k]$,  we have $f(x)_j = \sqrt{k}$ when $x_{2:d}$ viewed as a binary number equals to $j$, otherwise $f(x)_j = 0$. It can be verified that $\loss_{\lambda}(f) = 0$ hence $f$ is a minimizer of the contrastive loss. However, $f(x)$ is agnostic to the first dimension of $x$, hence the downstream error is at least $1$.
\end{proof}

\subsection{Proof for Example~\ref{example2}}\label{appendix:example2}

\begin{proof}[Proof of Theorem~\ref{theorem:example2}]
	For any vector $h\in\{-1, 1\}^s$, we define function $\binary(h)\in\{0, 1, \cdots, 2^s-1\}$ be the function that maps $h$ to the corresponding number when viewing $\frac{1}{2}(h+1)$ as binary.  Since $\binary(\cdot)$ is a one-to-one mapping, we can define $U \in\{k\times d\}$ such that the $i$-th row of $U$ satisfies: the first $s$ dimensions equal to $\sqrt{k} \cdot \binary^{-1}(i-1)$, and the rest $d-s$ dimensions are $0$. Let bias vector $b\in\real^{k}$ such that every dimension is $-\sqrt{k}\cdot (r-1)$. We have $f_{U, b}(x) = \sqrt{k} \cdot e_{\binary(x_{1:s})+1}\in\real^{k}$.
	
	Since $\Exp_{x\sim\pdata}[f_{U, b}(x)f_{U, b}(x)^\top] = \identity$ and $\Exp_{(x, x^+)\sim\ppos}[\norm{f_{U, b}(x) - f_{U, b}(x^+)}_2^2] = 0$, we can view $f_{U, b}$ as the $\feig$ in Section~\ref{section:eigenspace}. Assumption~\ref{assumption:downstream_2} naturally hold wihth $B=1$. 
	
	For Assumption~\ref{assumption:no_other_eigenvector_2}, consider a function $f_{U', b'}\in\reluf$ and index $i\in[k]$ such that $\Exp_{(x, x^+)\sim\ppos}\left[(f_{U', b'}(x)_i - f_{U', b'}(x^+)_i)^2\right]=0$. Suppose there exist $\bar{x}\ne \bar{x}'$ and their augmentations $x, x'$ such that $f_{U', b'}(x)_i > f_{U', b'}(x')_i$. Then, there must be ${(U_{i})}_{r+1:d}\ne 0$ and $\sigma(U_{i}^\top {x}) > 0$. This suggests that there must exist another $\tilde{x}$ which is also an augmentation of $\bar{x}$ but $\sigma(U_{i}^\top {x}) \ne \sigma(U_{i}^\top \tilde{x})$. Hence, we have
	\begin{align}
		\Exp_{(x, x^+)\sim\ppos}\left[(f_{U', b'}(x)_i - f_{U', b'}(x^+)_i)^2\right] >0,
	\end{align}
	leading to contradiction. Hence, we know that $f_{U', b'}(x)_i = f_{U', b'}(x')_i$, so $(f_{U', b'})_i$ can only be a function of $x_{1:s}$. Therefore, there exists a vector $w\in\real^k$ such that $f_{U', b'}(x)_i = w^\top f_{U, b}(x)$, which means Assumption~\ref{assumption:no_other_eigenvector_2} holds with $\epsilon=0$. Applying Theorem~\ref{theorem:thm_eigenspace} finishes the proof for \eqref{eqn:6}.
	
	The result about universal function approximators follows the same proof as for Theorem~\ref{theorem:example1} execpt for constructing the function using the last $(d-s)$ dimensions rather than the last $(d-1)$ dimensions.
\end{proof}

\subsection{Proof for Example~\ref{example3}}\label{appendix:example3}

\begin{proof}[Proof of Theorem~\ref{theorem:example3}]
	Let $\id_x$ be the index such that $x\in S_{\id_{x}}$, and define function $\feig(x) = \sqrt{m}\cdot e_{\id_x}$. It can be verified that $\feig$ satisfies \eqref{eqn:7} and \eqref{eqn:8}. For $f\in\lipf$ and $i\in[m]$, define $g(x) = f(x)_i$. 
	Suppose $\Exp_{x\sim\pdata}[g(x)^2]=1$, we can chooose $m$ data $x_1, x_2, \cdots, x_m$ such that $x_i \in S_i$ and $\frac{1}{m}\sum_{i\in[m]} g(x_i)^2\le 1$.		
	Define vector $\tilde{w}\in\real^m$ such that $\tilde{w}_i = \frac{1}{\sqrt{m}}\cdot g(x_i)$. We have
	\begin{align}
		\Exp_{x\sim\pdata}\left[(\tilde{w}^\top \feig(x) - g(x))^2\right] &= \frac{1}{m}\sum_{i\in[m]} \Exp_{x\sim S_i}	\left[(g(x_i) - g(x))^2\right]	\\
		&\le \frac{1}{m}\sum_{i\in[m]} \kappa^2\rho^2 = \kappa^2\rho^2.
	\end{align}
	Thus, $\feig$ satisfies Assumption~\ref{assumption:no_other_eigenvector_2} with $\epsilon = \kappa^2\rho^2$.
	
	Since the data in the same $S_i$ have the same downstream label, we know that Assumption~\ref{assumption:downstream_2} holds with $B=\sqrt{r}$ and $\alpha = 0$. Thus, applying Theorem~\ref{theorem:thm_eigenspace} finishes the proof for the upper bound.
	
	For the lower bound, Let set $\tilde{S}$ be the union of the $\frac{1}{2}r_0$ sets among those $r_0$ clusters that have the largest sizes. When $k\le\frac{1}{2}\cdot r_0$, we can construct a function that maps all data in $\tilde{S}$ to $0$, hence the final error would be at least $\frac{1}{2}$.
\end{proof}

\subsection{Proof for Example~\ref{example4}}\label{appendix:example4}

\begin{proof}[Proof of Theorem~\ref{theorem:example4}]
For any vector $h\in\{-1, 1\}^s$, we define function $\binary(h)\in\{0, 1, \cdots, 2^s-1\}$ be the function that maps $h$ to the corresponding number when viewing $\frac{1}{2}(h+1)$ as binary.  Since $\binary(\cdot)$ is a one-to-one mapping, we can define $U \in\{k\times s\}$ such that the $i$-th row of $U$ equal to $\frac{\sqrt{k}}{\gamma-1}\cdot \binary^{-1}(i-1)$. Let bias vector $b\in\real^{k}$ be such that every dimension is $\frac{\sqrt{k}}{\gamma-1}(-s-(\gamma-1)(s-1))$. We have $f^{\text{conv}}_{U, b}(x) =\sqrt{k} \cdot e_{\binary(x_{t:t+s-1})+1}\in\real^{k}$, where $t$ is the starting position of the informative patch in $x$. It can be verified that $\Exp_{x\sim\pdata}[f^{\text{conv}}_{U, b}(x)f^{\text{conv}}_{U, b}(x)^\top] = \identity$ and $\Exp_{(x, x^+)\sim\ppos}[\norm{f^{\text{conv}}_{U, b}(x) - f^{\text{conv}}_{U, b}(x^+)}_2^2] = 0$. Also, Assumption~\ref{assumption:downstream_2} holds with $B=1$ when viewing $\feig = f^{\text{conv}}_{U, b}$.

Suppose some function $f^{\text{conv}}_{\hat{U}, \hat{b}}$ and dimension $i\in[k]$ satisfies
\begin{align}
	\Exp_{(x, x^+)\sim\ppos}\left[\left(f^{\text{conv}}_{\hat{U}, \hat{b}}(x)_i - f^{\text{conv}}_{\hat{U}, \hat{b}}(x^+)_i\right)^2\right] = 0.
\end{align}
Then, we know that for any $x\in\pdata$, suppose we define $\tilde{x}$ as the vector that replaces spurious dimensions of $x$ with $0$. Notice that $\tilde{x}$ is in the support of $x$'s augmentations, and the model is continuous, we know $f^{\text{conv}}_{\hat{U}, \hat{b}}(x)_i = f^{\text{conv}}_{\hat{U}, \hat{b}}(\tilde{x})_i$
Further notice that for any two data $x, x'$ with the same informative patch (location might be different) and corresponding $\tilde{x}, \tilde{x'}$, there must be $f^{\text{conv}}_{\hat{U}, \hat{b}}(\tilde{x})_i  = f^{\text{conv}}_{\hat{U}, \hat{b}}(\tilde{x}')_i $ due to the structure of the convolutional neural networks. Thus, We have $f^{\text{conv}}_{\hat{U}, \hat{b}}({x})_i  = f^{\text{conv}}_{\hat{U}, \hat{b}}({x}')_i $. This suggests that the funciton $f^{\text{conv}}_{\hat{U}, \hat{b}}(\cdot)_i $ is in the span of $f^{\text{conv}}_{{U}, {b}}$, hence finishes the proof for the upper bound.

For the lower bound, we note that due to the lack of invariance to informative patch location, we can construct a network with $d\cdot 2^s$-dimensional output that satisfies \eqref{eqn:7} and \eqref{eqn:8}. When then output dimension is less than $d\cdot 2^{s-1}$, there would exist a minimizer of the contrastive loss that merges $d\cdot 2^{s-1}$ pairs of clusters. If every pair of clusters has different downstream labels, there would be at least $\frac{1}{2}$ loss incurred due to the data being mapped to the same feature, hence finishes the proof for the lower bound.
\end{proof}

\end{document}